\documentclass[12pt]{article}

\usepackage{a4wide}
\usepackage[latin1]{inputenc}
\usepackage{amsfonts}
\usepackage{amsmath}
\usepackage{amssymb}
\usepackage{amsthm}
\usepackage{caption}
\usepackage{fontenc}
\usepackage{epic}
\usepackage{float}
\usepackage{verbatim}
\usepackage{hyperref}
\usepackage{tikz}

\newenvironment{notation}{\vspace{0.3cm} \noindent {\bf Notation:\ }}{\vspace{0.3cm}}
\newenvironment{example}{\begin{exmp} \rm }{\hfill $\Box$ \end{exmp}}
\newtheorem{theorem}{Theorem}[section]
\newtheorem{definition}[theorem]{Definition}
\newtheorem{proposition}[theorem]{Proposition}
\newtheorem{lemma}[theorem]{Lemma}
\newtheorem{corollary}[theorem]{Corollary}
\newtheorem{exmp}[theorem]{Example}
\newtheorem{conjecture}[theorem]{Conjecture}

\floatstyle{ruled}
\newfloat{pattern}{h}{loa}
\floatname{pattern}{Pattern}

\DeclareMathOperator{\occursin}{\rightarrow}
\DeclareMathOperator{\notoccursin}{\not\rightarrow}
\newcommand{\CSP}[1]{\ensuremath{\operatorname{CSP}(#1)}}
\newcommand{\cspneg}[1]{\ensuremath{\CSP{\overline{#1}}}}
\newcommand{\unstructured}{flat}

\newcommand{\unknown}{U}
\newcommand{\var}[1]{v_{#1}}
\newcommand{\varb}[1]{w_{#1}}
\newcommand{\varc}[1]{x_{#1}}

\newcommand{\tuple}[1]{\ensuremath{\left\langle {#1} \right\rangle}}
\newcommand{\dotcup}{\ensuremath{\mathaccent\cdot\cup}}
\newcommand{\defterm}[1]{\textbf{#1}}
\newcommand{\tree}{{\text{\sc Tree}}}
\newcommand{\maxbin}{\text{\sc max2}}
\newcommand{\btp}{\text{\sc btp}}
\newcommand{\alldiff}{\text{\sc AllDifferent}}
\newcommand{\unary}{\text{\sc unary}}
\newcommand{\negtrans}{\text{\sc Negtrans}}
\newcommand{\SepPivot}{\textnormal{{\sc SepPivot}}}
\newcommand{\Pivot}{\textnormal{{\sc Pivot}}}
\newcommand{\Cycle}{\text{\sc Cycle}}
\newcommand{\Valency}{\text{\sc Valency}}
\newcommand{\Path}{\text{\sc Path}}

\title{
The tractability of CSP classes defined by forbidden patterns}
\author{David A. Cohen \\
  \small Royal Holloway, University of London,\\
  \small \texttt{dave@cs.rhul.ac.uk}\\
  \and
  Martin C. Cooper \\
  \small IRIT, University of Toulouse III, \\
  \small \texttt{cooper@irit.fr}\\
  \and
  P\'aid\'i Creed \\
  \small Royal Holloway, University of London,\\
  \small \texttt{paidi@cs.rhul.ac.uk}\\
  \and
  Andr\'as Z. Salamon \\
  \small University of Edinburgh,\\
  \small \texttt{asalamon@inf.ed.ac.uk}\\
}
\date{\today}

\begin{document} 

\maketitle

\begin{abstract}
The constraint satisfaction problem (CSP) is a general problem
central to computer science and artificial intelligence. Although the
CSP is NP-hard in general, considerable effort has been spent on
identifying tractable subclasses. The main two approaches consider
{\em structural} properties (restrictions on the hypergraph of
constraint scopes) and {\em relational} properties (restrictions on
the language of constraint relations). Recently, some authors have
considered {\em hybrid} properties that restrict the constraint
hypergraph and the relations simultaneously.

Our key contribution is the novel concept of a {\em CSP pattern} and
classes of problems defined by {\em forbidden patterns} (which can be
viewed as forbidding generic sub-problems). We describe the
theoretical framework which can be used to reason about classes of
problems defined by forbidden patterns. We show that this framework
generalises relational properties and allows us to capture known
hybrid tractable classes.

Although we are not close to obtaining a dichotomy concerning the
tractability of general forbidden patterns, we are able to make some
progress in a special case: classes of problems that arise when we
can only forbid binary negative patterns (generic sub-problems in
which only inconsistent tuples are specified).  In this case we are
able to characterise very large classes of tractable and NP-hard
forbidden patterns.  This leaves the complexity of just one case
unresolved and we conjecture that this last case is tractable.

\end{abstract}


\section{Introduction}
In the constraint satisfaction paradigm we consider computational
problems in which we have to assign values (from a \textit{domain})
to \textit{variables}, under some \textit{constraints}.  Each
constraint limits the (simultaneous) values that a list of variables
(its \textit{scope}) can be assigned.  In a typical situation some
pair of variables might represent the starting times of two jobs in a
machine shop scheduling problem.  A reasonable constraint would
require a minimum time gap between the values assigned to these two
variables.

Constraint satisfaction has proved to be a useful modelling tool in a
variety of contexts, such as scheduling, timetabling, planning,
bio-informatics and computer vision. This has led to the development
of a number of successful constraint solvers. Unfortunately, solving
general constraint satisfaction problem (CSP) instances is NP-hard
and so there has been significant research effort into finding
tractable fragments of the CSP.

In principle we can stratify the CSP in two quite distinct and
natural ways. The structure of the constraint scopes of an instance
of the CSP can be thought of as a hypergraph where the variables are
the vertices, or more generally as a relational structure.  We can
find tractable classes by restricting this relational structure,
while allowing arbitrary constraints on the resulting scopes
\cite{Dechter1987:network}. Sub-problems of the general constraint
problem obtained by such restrictions are called structural.
Alternatively, the set of allowed assignments to the variables in the
scope can be seen as a relation. We can choose to allow only
specified kinds of constraint relations, but allow these to interact
in an arbitrary structure \cite{Jeavons1997:closure}. Such restrictions
are called relational or language-based.

Structural subclasses are defined by specifying a set of hypergraphs
(or relational structures) which are the allowed structures for CSP
instances.
It has been shown that tractable structural classes are characterised
by limiting appropriate (structural) width
measures
\cite{Dechter1989:tree,Freuder1990:complexity,Gyssens1994:decomposing,Gottlob2002:hypertree,Marx2010:tractable}. 
For example, a  tractable structural class of binary CSPs is obtained
whenever we restrict the constraint structure (which is a graph in
this case)  to have bounded tree
width~\cite{Dechter1989:tree,Freuder1990:complexity}. In fact, it has been
shown that, subject to certain complexity-theoretic assumptions, the
only structures which give rise to tractable CSPs are those with
bounded (hyper-)tree
width~\cite{Dalmau2002:constraint,Grohe2006:structure,Grohe2007:complexity,Marx2010:tractable}.  

Relational subclasses are defined by specifying a set of constraint
relations.  It has been shown that the complexity of the subclass
arising from any such restriction is precisely determined by the so
called {\em polymorphisms} of the set of
relations~\cite{Bulatov2005:classifying,Cohen2006:complexity}\nocite{Rossi2006:handbook}.
The polymorphisms specify that, whenever some set of tuples is in a
constraint relation, then it cannot be the case that a particular
tuple (the result of applying the polymorphism) is not in the
constraint relation.  It is thus the relationship between allowed
tuples and disallowed tuples inside the constraint relations that is
of key importance to the relational tractability of any given class
of instances.
Whilst a general dichotomy has not yet been proven for the relational
case, many dichotomies on sub-problems have been obtained,
e.g.~\cite{Bulatov2003:tractable,Bulatov2005:classifying,Bulatov2006:dichotomy}.  

Unfortunately, by allowing only structural or relational restrictions
we limit the possible subclasses that we can define. By
considering restrictions on both the structure of the constraint graph
and the relations, we are able to identify new tractable classes. We call
these restrictions {\em hybrid} reasons for tractability.

Several hybrid results have been published for binary
CSPs~\cite{Jegou1993:decomposition,Cohen2003:new,Salamon2008:perfect:single,Cooper2010:generalizing,Cooper2010:hybrid}.
Instead of looking at the constraint graph or the constraint
language, these works captured tractability based upon the properties
of the (coloured) microstructure of CSP instances. The {\em
microstructure} of a binary CSP instance is the graph $\tuple{V,E}$
where $V$ is the set of possible assignments of values to variables
and $E$ is the set of pairs of mutually consistent variable-value
assignments~\cite{Jegou1993:decomposition}. In the {\em coloured
microstructure}, the vertices representing an assignment to variable
$v_i$ are labelled by a colour representing variable $v_i$, thus
maintaining the distinction between assignments to different
variables.

The coloured microstructure of a CSP instance captures both the
structure and the relations of a CSP instance and so it is a natural
place to look for tractable classes which are neither purely
structural nor purely relational.

Of the results on (coloured) microstructure properties, two are of
particular note. First it was observed that the class of instances
with a perfect microstructure is
tractable~\cite{Salamon2008:perfect:single}.
This is a proper generalisation of the well known hybrid tractable
CSP class whose instances allow arbitrary unary constraints and in
which every pair of variables is constrained to be not
equal~\cite{Regin1994:filtering,vanHoeve2001:alldifferent}, and of the hybrid
class whose microstructure is
triangulated~\cite{Jegou1993:decomposition,Cohen2003:new}.  
The perfect microstructure property excludes an infinite set of
induced subgraphs from the microstructure. In this paper, we provide
a different hybrid class that also strictly generalises the class of
CSP instances with an inequality constraint between every pair of
variables and an arbitrary set of unary constraints, but does so by
forbidding a single pattern.

Secondly, the so called broken-triangle property properly extends the
structural notion of acyclicity to a more interesting hybrid
class~\cite{Cooper2010:generalizing}. The broken triangle property is specified
by excluding a particular pattern (a subgraph) in the coloured
microstructure. It is this property that we generalise in this paper.
We do this by working directly with the CSP instance (or equivalently
its coloured microstructure) rather than its microstructure
abstraction which is a simple graph. This allows us to introduce a
language for expressing hybrid classes in terms of forbidden
patterns, so obtaining novel hybrid tractable classes.  In the case
of binary negative patterns we are able to characterise very large
classes of tractable and NP-hard forbidden patterns.  This leaves the
complexity of just one case unresolved and we conjecture that this
last case is tractable, which would give us a new CSP dichotomy for
hybrid classes of binary CSPs defined by negative patterns.

\subsection*{Contributions}
In this paper we generalise the definition of a CSP instance to that
of a CSP pattern which has two types of tuple in its constraint
relations, tuples which are explicitly allowed/disallowed and tuples
which are labelled as {\em unknown}\footnote{This
  can be viewed as the natural generalisation of the CSP to a
  three-valued logic.}. By defining a
natural notion of containment of patterns in a CSP, we are able
to describe problems defined by {\em forbidden patterns}: a CSP
instance $P$ forbids a particular pattern if this pattern cannot be
contained in $P$. By defining problems in this way, we can capture
both allowed and disallowed constraint tuples as well as structural
properties. We use this framework to capture tractability by
identifying local patterns of allowed \textit{and} disallowed tuples
(within small groups of connected constraints) whose absence is
enough to guarantee tractability.

By extending the notion of the CSP instance to that of a pattern we
are able to unify the following properties:
\begin{itemize}
\item having a particular polymorphism;
\item having a hereditary (coloured) microstructure property, such as broken
  triangle~\cite{Cooper2010:generalizing}; and
\item having a tree structure (tree width $1$).
\end{itemize}

Using the concept of forbidden patterns, we lay foundations for a
theory that can be used to reason about classes of CSPs defined by
hybrid properties.
Since this is the first work of this kind, we primarily focus on the
simplest case: binary patterns in which tuples are either disallowed
or unknown (called {\em negative  patterns}). We give a large class
of binary negative patterns which give rise to intractable classes of
problems and, using this, show that any negative pattern that defines
a tractable class of problems must have a certain structure. We are
able to prove that this structure is nearly enough to guarantee
tractability and conjecture that there is a precise condition
providing dichotomy for tractability defined by forbidding binary
negative patterns. Importantly, our intractability results also allow
us to give a necessary condition on the form of general tractable
patterns.

The remainder of the paper is structured as follows. In
Section~\ref{sec:prelim} we define constraint satisfaction problems,
and give other definitions which will be used in the paper. Then, in
Section~\ref{sec:forbid}, we define the notion of a CSP pattern and
describe classes of problems defined by forbidden patterns. In
Section~\ref{sec:nomaximal} we show that one must take the size of
patterns into account to have a notion of maximal classes defined by
forbidding patterns. In Section~\ref{sec:examples} we give some
examples of tractable classes defined by forbidden patterns on three
variables. In general, we are not yet able to make any conjecture
concerning a dichotomy for hybrid
tractability defined by general forbidden patterns. %
However, in Section~\ref{sec:pivot} we are able to give a necessary
condition for such a class %
to be tractable.
Finally, in Section~\ref{sec:conclusions} we summarise the
results of this work and discuss directions for future research.

\section{Preliminaries}\label{sec:prelim}

\begin{definition}
A CSP instance is a triple \tuple{V,C,D} where
\begin{itemize}
 \item $V$ is a finite set of \textbf{variables} (with $n=|V|$).
 \item $D$ is a finite set called the \textbf{domain} (with $d=|D|$).
 \item $C$ is a set of \textbf{constraints}.  Each constraint $c \in
   C$ is a pair $c = \tuple{\sigma,\rho}$ where
 \begin{itemize}
 \item $\sigma$ is a list of variables called the \textbf{scope} of $c$.
 \item $\rho$ is a relation over $D$ of arity $|\sigma|$  called the
   \textbf{relation} of $c$. It is the set of tuples allowed by $c$.
 \end{itemize}
\end{itemize}

A \textbf{solution} to the CSP instance $P = \tuple{V,D,C}$ is a
mapping $s:V \mapsto D$ where, for each $\tuple{\sigma,\rho} \in C$
we have $s(\sigma) \in \rho$ (where $s(\sigma)$ represents the tuple
resulting from the application of $s$ component-wise to the list of
variables $\sigma$).
\end{definition}

The arity of a CSP is the largest arity
of any of its constraint scopes. Our long-term aim is to identify all
tractable subclasses of the CSP problem which can be detected in
polynomial time.  In this paper we describe a general theory of
forbidden patterns for arbitrary arity but only consider the
implications of the new theory for tractable classes of arity two (binary)
problems specified by finite sets of forbidden patterns.
In such cases we are certain that class membership can be
decided in polynomial time.

The CSP decision problem, which asks whether a particular CSP instance has a
solution, is already NP-complete for binary CSPs. For example, there
is a straightforward reduction from graph colouring to this problem
in which vertices $i$ of the graph map to CSP variables $\var{i}$ and
edges $\{i,j\}$ map to disequality constraints $\var{i} \neq \var{j}$.

It will sometimes be convenient in this paper to use an equivalent
functional formulation of a constraint.  In this alternative
formulation the scope $\sigma$ of the constraint \tuple{\sigma,\rho}
is abstracted to a set of variables and each possible assignment is
seen as a function  $f : \sigma \mapsto D$. 
The constraint relation in this alternative view is then a function
from the set of possible assignments, $D^\sigma$, into the set $\{T,F\}$
where, by convention, the tuples which occur in the constraint relation
are those which map to $T$. It follows that any assignment to the
set of all variables is allowed by \tuple{\sigma,\rho} when its
restriction to $\sigma$ is mapped to $T$ by $\rho$.

\begin{definition}
For any function $f : X \mapsto Y$ and $S \subset X$, the notation
$f|_S$ means the function with domain $S$ satisfying $f|_S(x)=f(x)$
  for all $x \in S$.

Given a set $V$ of variables and a domain $D$, a \textbf{constraint}
in functional representation is a pair $\tuple{\sigma,\rho}$ where
$\sigma \subseteq V$ and $\rho: D^\sigma \mapsto \{T,F\}$. A
\textbf{CSP instance} in functional representation is a triple
$\tuple{V,D,C}$ where $C$ is a set of constraints in functional
representation.

A \textbf{solution} (to a CSP instance $\tuple{V,D,C}$ in functional
representation) is a mapping $s:V \mapsto D$ where, for each
$\tuple{\sigma,\rho} \in C$ we have $\rho(s|_\sigma) = T$.
\end{definition}

The functional formulation is clearly equivalent to the relational
formulation and we will use whichever seems more appropriate
throughout the paper.  The choice will always be clear from the
context.

\subsection*{Relational tractability of binary CSP}

We will refer to a set of relations $\Gamma$ on some finite set $D$
as a {\em constraint language}.

\begin{definition}
Let $D$ be a finite set and let $\Gamma$ be a set of relations on
$D$. We define $\CSP{\Gamma}$ to be the set of problems for which
every constraint \tuple{\sigma,\rho} satisfies $\rho \in \Gamma$.
\end{definition}

A constraint language $\Gamma$ is said to be {\em tractable} if
\CSP{\Gamma'} is a tractable class of problems for each finite
$\Gamma' \subseteq \Gamma$. It is well-known that the tractability of
$\Gamma$ can be determined by studying the {\em polymorphisms} of
$\Gamma$~\cite{Cohen2006:complexity}.
\begin{definition} \label{defn:poly}
Let $D$ be a finite set and let $\rho$ be a binary relation on $D$. A
$k$-ary \textbf{polymorphism} of $\rho$ is a function $f : D^k
\mapsto D$ satisfying
\[
\forall x_1,\ldots,x_k \in \rho,\quad
\tuple{f(x_1[1],\ldots,x_k[1]),f(x_1[2],\ldots,x_k[2])} \in \rho\,.
\]
\end{definition}

It is known that the existence of a non-trivial polymorphism is a
necessary condition for a set of relations to give rise to a
tractable constraint language
\cite{Jeavons1997:closure,Cohen2006:complexity,Bulatov2005:classifying}. Using
this characterisation, almost all tractable classes of the CSP defined
by sets of relations have been determined, though establishing the
full dichotomy still remains an open problem.

\subsection*{Structural tractability of binary CSP}
Structural tractability considers the classes of problems defined by
placing restrictions on the set of constraint scopes, but which allow
arbitrary constraint relations. For simplicity, and as this is the
focus of this paper, we restrict our attention to binary CSPs. In
this case, the set of constraint scopes defines the {\em constraint
graph} whose vertices are the variables and whose edges are the set
of scopes of constraints whose relation is not complete (i.e.~the
Cartesian product $D^{2}$). All definitions and concepts extend
naturally to non-binary CSPs. The key property here is the {\em tree
width} of the constraint graph.

\begin{definition}
Let $G$ be a graph. A {\em tree decomposition} of $G$ is a pair
$(T,X)$, where $T$ is a tree and $X$ is a mapping that associates with
every node $t \in V(T)$ a set $X_t \subset V(G)$ such that for every
$v \in V(G)$ the set $\{t \in V(T) \mid v \in X_t\}$ is connected, and
for every $e \in E(G)$ there is a $t \in V(T)$ such that $e \subset
X_t$.

The {\em width} of a tree decomposition $(T,X)$ is $\max\{|X_t|-1 \mid
t \in V(T)\}$. The {\em tree width} of a graph $G$, denoted $tw(G)$,
is the minimum of the widths of all tree decompositions of $G$.
\end{definition}

The following theorem is classical
\cite{Dechter1989:tree,Freuder1990:complexity}.
\begin{theorem}
\label{thm:tw} Let $P$ be a CSP. If the constraint graph of $P$,
$G_P$, has $tw(G_P) \leq k$, then we can solve $P$ in time $O(nd^{k+1})$.
\end{theorem}

What is more, under reasonable technical assumptions, there is no
property of $G_P$ which gives rise to a larger tractable class of
CSPs. This establishes a dichotomy for structural tractability of
binary CSPs. A similar result has been obtained for CSPs of higher
arity. 
See
\cite{Dalmau2002:constraint,Grohe2006:structure,Grohe2007:complexity,Marx2010:tractable}
for more details.  

\section{Forbidden patterns in CSP}
\label{sec:forbid}

In this paper we explain how we can define classes of CSP instances
by forbidding the occurrence of certain patterns. A CSP pattern is a
generalisation of a CSP instance.  In a CSP pattern we define the
relations relative to a three-valued logic on $\{T,F,U\}$, meaning
that the pattern is simply the set of CSP instances in which each
undefined value $U$ is replaced by either $T$ or $F$. Forbidding a
CSP pattern is equivalent to simultaneously forbidding all these
instances as sub-problems.
\begin{definition}
We define a three-valued logic on $\{T,F,\unknown\}$, where
$\unknown$ stands for \emph{unknown} or \emph{undefined}. The set
$\{T,F,\unknown\}$ is partially ordered so that $\unknown < T$ and
$\unknown < F$ but $T$ and $F$ are incomparable.

Let $D$ be a finite set. A $k$-ary \textbf{three-valued relation}
on $D$ is a function $\rho : D^k \mapsto
\{T,F,\unknown\}$. Given a pair of $k$-ary three-valued relations
$\rho$ and $\rho'$, we say $\rho$ \textbf{realises}
$\rho'$ if
\[
\forall x \in D^k, \rho(x) \geq \rho'\,.
\]

\end{definition}
\begin{definition}
A \textbf{CSP pattern}
is a triple $\chi=\tuple{V,D,C}$ where:
\begin{itemize}
\item $V$ is the set of \textbf{variables}.
\item $D$ is the \textbf{domain}.
\item $C$ is a set of constraint \textbf{patterns}.  Each constraint pattern
  $c \in P$ is a pair $c = \tuple{\sigma,\rho}$,
  where $\sigma \subseteq V$, the \textbf{scope} $\sigma$ of $c$, is a set of
  variables and $\rho: D^\sigma \mapsto \{T,F,\unknown\}$ is the
  \textbf{three-valued relation} (in functional representation) of $c$.
\end{itemize}

\noindent
The \textbf{arity} of a CSP pattern $\chi$ is the maximum arity of
any constraint pattern $\tuple{\sigma,\rho}$ of $\chi$.
\end{definition}

We will sometimes define $\rho$ as a partial function from $D^\sigma$
to $\{T,F\}$; the values for which $\rho$ is undefined are those
which are mapped to $U$. For simplicity of presentation, we assume
throughout this paper that no two constraint patterns in $C$ have the
same scope (and that, in the case of CSP instances, that no two
constraints have the same scope).
We will represent binary CSP
patterns by simple diagrams. Each oval represents the domain of a
variable, each dot a domain value. The tuples in constraint patterns
with value $F$ are shown by dashed lines, those with value $T$ by
solid lines and those with value $\unknown$ are not depicted at all.

\subsection*{Contexts}
By further adding simple structure to the domains and variable sets
of patterns, we are able to make the notion of patterns more
specific, and so we can capture larger, and more interesting,
tractable classes. Contexts such as these have been used in the past
to capture tractable classes.  For example, when the domain is
totally ordered we can define the tractable max-closed
class~\cite{Jeavons1995:tractable}, and when we have an independent
total order for the domain of each variable we can capture the
renamable Horn class~\cite{Green2003:tractability}.

The weakest such context that we will consider only allows us to say
when two variables are distinct. A pattern with such a context will
be called \textbf{\unstructured}. We give a general definition of
context, but in this paper the only contexts we require are total
orders on the variable set or the domain.

\begin{definition}
A \textbf{CSP context} is a set of relational structures
$\Omega$ on the product of the variable set and domain\footnote{We
tacitly assume that a context contains at least one structure for
every combination of sizes of variable set and domain.}.

A CSP pattern \tuple{V,D,C} is considered in context $\Omega$ by
associating it with a structure $\omega \in \Omega$ for
appropriately-sized variable set and domain.

Let \tuple{V,D} and \tuple{V',D'} be in context $\Omega$, with
$\omega$ and $\omega'$ the elements of $\Omega$ giving structure to
the sets $V \times D$ and $V' \times D'$, respectively. A
\textbf{contextual homomorphism} is an $\Omega$-structure preserving
function $F : V \times D \mapsto V' \times D'$, i.e.~for each $(u,a),
(v,b) \in V \times D$, $\tuple{(u,a),(v,b)} \in \omega$ implies
$\tuple{F(u,a),F(v,b)} \in \omega'$.   
\end{definition}

\begin{definition}
Two CSP patterns are \textbf{compatible} if they are considered in the
same context.
\end{definition}

Thus, for example, two CSP patterns with totally-ordered domains are
compatible even if the domain sizes are different. In this case, a
contextual homomorphism between the two patterns must preserve the
domain ordering.

\subsection*{Patterns, CSPs and occurrence}
A CSP instance is just a CSP pattern in which the three-valued
relations of the constraint patterns never take the value $\unknown$.
That is, we decide for each possible tuple whether it is in the
relation or not.  %
Furthermore, in a CSP instance, for each pair of variables we assume
that a constraint exists with this scope; if no explicit constraint
is given on this scope, then we assume that the relation is complete,
i.e.~it contains all tuples. This can be contrasted with CSP patterns
for which the absence of an explicit constraint on a pair of
variables implies that the truth value of each tuple is undefined.

In order to define classes of CSP instances by forbidding patterns,
we require a formal definition of an occurrence (containment) of a pattern
within an instance. We define the more general notion of containment of
one CSP pattern within another pattern. Informally, the names of the
variables and domain elements of a CSP pattern are inconsequential
and a containment allows a renaming of the variables and the domain
values of each variable. Thus, in order to define the containment of
patterns, we firstly require a formal definition of a renaming. In an
arbitrary renaming two distinct variables may map to the same
variable and two distinct domain values may map to the same domain
value. However, we do not allow distinct variables $v_{1},v_{2}$ to
map to the same variable if $v_{1},v_{2}$ belong to the scope of any
binary constraint (otherwise this binary constraint could not be
correctly represented on a scope consisting of a single variable).

A domain labelling of a set of variables is just an assignment of
domain values to those variables. Variable and domain renaming
induces a mapping on the domain labellings of scopes of constraints:
we simply assign the renamed domain values to the renamed variables.
There is a natural way to extend this mapping of domain labellings to
a mapping of a constraint pattern: the truth-value of each mapped
domain labelling is the same as the truth-value of the original
domain labelling. However, it may occur that two domain labellings of
some scope map to the same domain labelling, so instead the resulting
value is taken to be the greatest of the original truth-values. (In
order for this process to be well-defined, if two domain labellings
of a constraint are mapped to the same domain labelling, then their
original truth-values must be comparable.) This leads to the
following formal definition of a renaming which is the first step
towards the definition of containment.

\begin{definition}
Let $\chi = \tuple{V,D,C}$ and $\chi' = \tuple{V',D',C'}$ be
compatible CSP patterns.

We say that $\chi'$ is a \defterm{renaming} of $\chi$ if there exist
a variable renaming function $s: V \mapsto V'$ and a domain renaming
function $t: V \times D \mapsto D'$ that satisfy:

\begin{itemize}
\item If $s(v_{1})=s(v_{2})$ for distinct variables $v_{1},v_{2}$,
then there is no constraint pattern $\tuple{\sigma,\rho} \in C$ with
$v_{1},v_{2} \in \sigma$ and $\rho$ a non-trivial relation (a
function which is not identically equal to $U$).
\item $F:V \times D \mapsto V' \times D'$ defined by $F(\tuple{v,a}) =
  \tuple{s(v),t(v,a)}$ is a contextual homomorphism.
\item
For each constraint pattern $\tuple{\sigma,\rho} \in C$, for any two
domain labellings $\ell,\ell' \in D^\sigma$ for which $F(\ell) = F(\ell')$, we
have that $\rho(\ell)$ and $\rho(\ell')$ are comparable.
\item
$C' = \{ \tuple{s(\sigma),\rho'} \mid \tuple{\sigma,\rho} \in
C\}$, where $\rho'(f) = \max \left\{\rho(\ell) \mid F(\ell) = f\right\}$,
for each $f : s(\sigma) \mapsto D$.
\end{itemize}
\end{definition}

We will use patterns to define sets of CSP instances by forbidding
the occurrence (containment) of the patterns in the CSP instances. In
this way we will be able to characterise tractable subclasses of the
CSP.   Informally, a pattern $\chi$ is said to occur in a CSP
instance $P$ if we can find variables in $P$ corresponding to the
variables of $\chi$, such that there is a constraint in $P$ that
realises each constraint pattern in $\chi$. We will now formally
define what we mean by a pattern occurring in another pattern (and by
extension, in a CSP instance).

\begin{definition}
\label{defn:embed} 
We say that a CSP pattern $\chi$ \defterm{occurs}
in a CSP pattern $P =\tuple{V,D,C}$ (or that $P$ \defterm{contains}
$\chi$), denoted $\chi \occursin P$, if there is a renaming
\tuple{V,D,C'} of $\chi$ where, for every constraint pattern
$\tuple{\sigma,\rho'} \in C'$ there is a constraint pattern
$\tuple{\sigma, \rho} \in C$ and, furthermore, $\rho$ realises
$\rho'$.
\end{definition}

\begin{pattern}\vspace{0.5cm}
\begin{center}
  \begin{tabular}{@{\extracolsep{1cm}}ccc}
    \begin{tikzpicture}%
      \path[draw, rounded corners=15pt] (0.5,3) rectangle (1.5,1);
      \fill (1,1.5) circle (0.1);
      \fill (1,2.5) circle (0.1);
      \node at (0.25,1.5) {$a$};
      \node at (0.25,2.5) {$b$};

      \path[draw, rounded corners=15pt] (0.5+2,3) rectangle (1.5+2,1); 
      \fill (1+2,1.5) circle (0.1);
      \fill (1+2,2.5) circle (0.1);
      \node at (0.25+3.5,1.5) {$c$};
      \node at (0.25+3.5,2.5) {$d$};    

      \draw (1,1.5) -- (1+2,2.5);
      \draw (1,2.5) -- (1+2,1.5);
      
      \node at (1,0.5) {$x$};
      \node at (1+2,0.5) {$y$};
    \end{tikzpicture}  
    &
    \begin{tikzpicture}%
      \path[draw, rounded corners=15pt] (0.5,3+0.5) rectangle (1.5,1+0.5);
      \fill (1,1.5+0.5) circle (0.1);
      \fill (1,2.5+0.5) circle (0.1);
      \node at (0.25,1.5+0.5) {$a$};
      \node at (0.25,2.5+0.5) {$b$};

      \path[draw, rounded corners=15pt] (0.5+2,4) rectangle (1.5+2,1); 
      \fill (1+2,1.5) circle (0.1);
      \fill (1+2,2.5) circle (0.1);
      \fill (1+2,3.5) circle (0.1);
      \node at (0.25+3.5,1.5) {$d$};
      \node at (0.25+3.5,2.5) {$c$};
      \node at (0.25+3.5,3.5) {$d'$};

      \draw[dashed] (1,3) -- (1+2,1.5);
      \draw (1,3) -- (1+2,2.5);
      \draw[dashed] (1,3) -- (1+2,3.5);
      
      \draw (1,2) -- (3,1.5);
      \draw (1,2) -- (3,3.5);
      
      \node at (1,0.5) {$x$};
      \node at (1+2,0.5) {$y$};
    \end{tikzpicture}  
    &
    \begin{tikzpicture}
      \path[draw] (1,2.75) circle (0.5);
      \fill (1,2.75) circle (0.1);
      \node at (0.25,2.75) {$b$};
      \node at (1,2) {$z$};

      \draw (1,1.25) circle (0.5);
      \fill (1,1.25) circle (0.1);
      \node at (0.25,1.25) {$a$};
      \node at (1,0.5) {$x$};
      
      \path[draw, rounded corners=15pt] (0.5+2,3) rectangle (1.5+2,1); 
      \fill (1+2,1.5) circle (0.1);
      \fill (1+2,2.5) circle (0.1);
      \node at (0.25+3.5,1.5) {$c$};
      \node at (0.25+3.5,2.5) {$d$};    
      \node at (1+2,0.5) {$y$};

      \draw (1,1.25) -- (1+2,2.5);
      \draw (1,2.75) -- (1+2,1.5);
      \draw[dashed] (1,2.75) -- (1+2,2.5);      
    \end{tikzpicture}  
    \\
    (i) & (ii) & (iii)
  \end{tabular}
\end{center}
\caption{}
\label{pat:embeddingex1}
\end{pattern}

\begin{pattern}\vspace{0.3cm}
\begin{center}
  \begin{tikzpicture}
    \path[draw, rounded corners=15pt] (0.5,3) rectangle (1.5,1);
      \fill (1,1.5) circle (0.1);
      \fill (1,2.5) circle (0.1);
      \node at (0.25,1.5) {$a$};
      \node at (0.25,2.5) {$b$};

      \path[draw, rounded corners=15pt] (0.5+2,3) rectangle (1.5+2,1); 
      \fill (1+2,1.5) circle (0.1);
      \fill (1+2,2.5) circle (0.1);
      \node at (0.25+3.5,1.5) {$c$};
      \node at (0.25+3.5,2.5) {$d$};    

      \draw (1,1.5) -- (1+2,2.5);
      \draw (1,2.5) -- (1+2,1.5);
      \draw[dashed] (1,2.5) -- (1+2,2.5);
     
      \node at (1,0.5) {$x$};
      \node at (1+2,0.5) {$y$};
  \end{tikzpicture}
\end{center}
\caption{}
\label{pat:embeddingex2}
\end{pattern}

\begin{example}
  This example describes three simple containments. Consider the
  three constraint patterns,
  Pattern~\ref{pat:embeddingex1}(i)-(iii). These
  constraint patterns occur in, or are contained in, Pattern~\ref{pat:embeddingex2} by the
  contextual homomorphisms $F_1$, $F_2$, and $F_3$, respectively, which
  we will now describe. 
  
  $F_1$ is simply a bijection. Although the patterns are different,
  this is a valid containment of Pattern~\ref{pat:embeddingex1}(i)
  into Pattern~\ref{pat:embeddingex2} because the three-valued relation
  of Pattern~\ref{pat:embeddingex2} is a realisation of the
  three-valued relation in Pattern~\ref{pat:embeddingex1}(i): we are
  replacing $(b,d) \rightarrow U$ by $(b,d) \rightarrow F$.

  $F_2$ maps $(x,a)$, $(x,b)$, and $(y,c)$ to themselves, and maps
  both $(y,d)$ and $(y,d')$ to $(y,d)$. This merging of domain
  elements is possible because the three-valued constraint relation of
  Pattern~\ref{pat:embeddingex1}(ii) agrees on tuples involving the
  assignments $(y,d)$ and $(y,d')$ and, furthermore, the restriction
  of the three-valued relation of Pattern~\ref{pat:embeddingex1}(ii)
  to either of these two assignments is equivalent to the three-valued
  constraint relation of Pattern~\ref{pat:embeddingex2}: $(b,d)
  \rightarrow F$ and $(a,d) \rightarrow T$. 
  
  Finally, $F_3$ maps $(y,c)$ and $(y,d)$ to themselves, and maps
  $(x,a)$ and $(z,b)$ in Pattern~\ref{pat:embeddingex1}(iii) to
  $(x,a)$ and $(x,b)$, respectively, in
  Pattern~\ref{pat:embeddingex2}. As before, this merging of variables
  is possible because the three-valued relations agree. 
\end{example}

Before continuing we need to define what we mean when we say that a
class of CSP instances is definable by forbidden patterns.
\begin{definition}
\label{defn:forbid} Let $C$ be any class of CSP instances.  We say
that $C$ is \defterm{definable by forbidden patterns} if there is
some context $\Omega$ and some set of patterns $\mathcal X$ for which the
set of CSP instances in which none of the patterns in $\mathcal X$ occur
are precisely the instances in $C$.
\end{definition}

\begin{notation}
Let $\mathcal X$ be a set of %
CSP patterns with maximum arity
$k$. We will use $\cspneg{\mathcal X}$ to denote the set of CSP instances
with arity at most $k$ (compatible with $\mathcal X$) in which no element
$\chi \in \mathcal X$ occurs.  When $\mathcal X$ is a singleton $\{\chi\}$ we will
allow \cspneg{\chi} to mean \cspneg{\{\chi\}}.
\end{notation}

\subsection*{Tractable Patterns}
In this paper we will define, by forbidding certain patterns,
tractable subclasses of the CSP. Furthermore, we will give examples
of truly hybrid classes (i.e.~not definable by purely relational or
purely structural properties).

\begin{definition}
A set of patterns $\mathcal X$ is \textbf{intractable} if $\cspneg{\mathcal
  X}$ is NP-hard. It is \textbf{tractable} if there is a
polynomial-time algorithm to solve $\cspneg{\mathcal X}$. A single pattern $\chi$
is tractable (intractable) if $\{\chi\}$ is tractable (intractable).
\end{definition}

It is worth observing that classes of CSP instances defined by
forbidding patterns do not have a fixed domain. Recall, however, that
CSP instances have finite domains.

We will need the following simple lemma for our proofs of intractability
results in later sections of this paper.

\begin{lemma}
\label{lem:embed}
Let $\mathcal X$ and $\mathcal T$ be sets of compatible CSP patterns and
suppose that for every pattern $\tau \in \mathcal T$, there is some pattern
$\chi \in \mathcal X$ for which $\chi \occursin \tau$.
Then $\cspneg{\mathcal X} \subseteq \cspneg{\mathcal T}$.
\end{lemma}
\begin{proof}
Let $P \in \cspneg{\mathcal X}$, so $\chi \notoccursin P$ for each $\chi
\in \mathcal X$. Then we cannot have $\tau \occursin P$ for any $\tau \in
\mathcal T$, since this would imply that there exists some $\chi \in \mathcal
X$ such that $\chi \occursin \tau \occursin P$ and hence that $\chi
\occursin P$. Hence, $P \in \cspneg{\mathcal T}$.
\end{proof}

\begin{corollary}
Let $\mathcal X$ and $\mathcal T$ be sets of compatible CSP patterns and
suppose that for every pattern $\tau \in \mathcal T$, there is some
pattern $\chi \in \mathcal X$ for which $\chi \occursin \tau$.

We then have that \cspneg{\mathcal T} is intractable if \cspneg{\mathcal X} is
intractable and conversely, that \cspneg{\mathcal X} is tractable whenever
\cspneg{\mathcal T} is tractable.
\end{corollary}

Finally, we give a very simple example of a tractable pattern. This
is an example of a {\em negative} pattern since the only truth-values
in the relations are $F$ and $U$.
We will use the notation NEQ($\var{1},\ldots,\var{r}$) to denote the
fact that the variables $\var{1},\ldots,\var{r}$ are all distinct.
\begin{pattern}\vspace{0.3cm}
\centering
\begin{tikzpicture}
  \node at (0.5,4.2) {$v$};
  \path[draw,rounded corners=15pt] (0,4) rectangle (1,2.5);
  \fill (0.5,3.25) circle (0.1);
  \node at (-0.3,3.25) {$a$};

  \node at (2,1.7) {$w$};
  \path[draw,rounded corners=15pt] (1.5,1.5) rectangle (2.5,0);
  \fill (2,0.75) circle (0.1);
  \node at (1.2,0.75) {$b$};

  \node at (4,3.2) {$x$};
  \path[draw,rounded corners=15pt] (3,3) rectangle (4,1.5);
  \fill (3.5,2.5) circle (0.1);
  \node at (4.3,2.5) {$c$};
  \fill (3.5,2) circle (0.1);
  \node at (4.3,2) {$c'$};

  \draw[dashed] (0.5,3.25) -- (3.5,2.5);
  \draw[dashed] (2,0.75) -- (3.5,2);

\end{tikzpicture}

NEQ($v,w,x$), $c \neq c'$
\caption{A very simple negative pattern.}
\label{pat:simple}
\end{pattern}

\begin{example}
Consider the pattern given as Pattern~\ref{pat:simple}. This defines
a class of CSPs which is trivially tractable. We can apply a
pre-processing step consisting of first establishing arc consistency,
and then assigning value $c$ to variable $x$ (and eliminating the
variable $x$) if this assignment is consistent with all remaining
assignments to other variables. Forbidding Pattern~\ref{pat:simple}
ensures that after this pre-processing step there are no paths of
length greater than $2$ in the constraint graph. Thus, any problem
forbidding Pattern~\ref{pat:simple} can be decomposed into a set of
independent sub-problems, each of maximum size $2$.
\end{example}

\subsection*{Relational and structural tractability as forbidden patterns}
The following examples show certain strengths %
of this notion for defining tractable classes. First, in
Example~\ref{ex:maxbin}, we show that forbidden patterns properly
generalise polymorphisms. Then, in Example~\ref{ex:tw}, we show that
we can define the class of tree-structured CSPs by a single forbidden
pattern.

\begin{example}
\label{ex:maxbin} Let $\tuple{D,<}$ be any totally ordered domain.  A
binary relation $\rho$ over $D$ is said to be {\em max-closed} if,
for any tuples $\tuple{a,b},\tuple{a',b'} \in \rho$ we have that
$\tuple{\max(a,a'),\max(b,b')} \in \rho$.  It is well known that the
class of CSP instances whose relations are binary max-closed is
tractable~\cite{Jeavons1995:tractable}. We can define the class of
max-closed CSPs as those forbidding the following pattern
(Pattern~\ref{pat:max2}):

\begin{itemize}
\item CSP context: the variable set is unstructured and the domain is
  totally ordered.
 \item Variables: $\{x,y\}$, where $x \neq y$.
\item Domain: The ordered set $\{0,1\}$ with $0 < 1$.
\item A single constraint pattern with scope $\{x,y\}$ and
  three-valued relation:
\begin{align*}
 \{x \mapsto 0, y \mapsto 0\} & \mapsto  \unknown\\
 \{x \mapsto 0, y \mapsto 1\} & \mapsto  T\\
 \{x \mapsto 1, y \mapsto 0\} & \mapsto  T\\
 \{x \mapsto 1, y \mapsto 1\} & \mapsto  F\\
\end{align*}
\end{itemize}
In this pattern, the context specifies that $x \neq y$ and $0 < 1$,
so we have limited the contextual homomorphisms to those that map $x$
and $y$ to distinct variables and $0$ and $1$ to ordered domain
values. Thus, saying that a pattern \maxbin\ is forbidden in a CSP
instance $P$ is equivalent to saying there is no constraint allowing
a pair of tuples $(a,b)$ and $(a',b')$, where $a < a'$ and $b' < b$,
such that $(a',b)$ is disallowed; this is equivalent to saying that
every constraint must be max-closed.
\end{example}

\begin{pattern}\vspace{0.3cm}
\centering
\begin{tikzpicture}
  \path[draw,rounded corners=15pt] (0,3) rectangle (1,1);
  \fill (0.5,1.5) circle (0.1);
  \fill (0.5,2.5) circle (0.1);

  \path[draw,rounded corners=15pt] (3,3) rectangle (4,1);
  \fill (3.5,1.5) circle (0.1);
  \fill (3.5,2.5) circle (0.1);

  \draw[dashed] (0.5,2.5) -- (3.5,2.5);
  \draw (0.5,2.5) -- (3.5,1.5);
  \draw (3.5,2.5) -- (0.5,1.5);

  \node at (0.5,0.5) {$x$};
  \node at (3.5,0.5) {$y$};

  \node at (-0.5,1.5) {$0$};
  \node at (-0.5,2.5) {$1$};
  \node at (4.5,1.5) {$0$};
  \node at (4.5,2.5) {$1$};
\end{tikzpicture}

Context: $x \neq y$, $0 < 1$

\caption{The \maxbin\ pattern.}
\label{pat:max2}
\end{pattern}

The set of max-closed relations are also known as the relations which
are closed under the polymorphism $\max(x,y)$.

Recall from Definition \ref{defn:poly} that, given a finite set $D$
and a binary relation $\rho$ on $D$, a $k$-ary polymorphism of
$\rho$ is a function $f : D^k \mapsto D$ satisfying
\[
\forall x_1,\ldots,x_k \in \rho,\quad f(x_1,\ldots,x_k) \in \rho\,,
\]
where
$f(x_1,\ldots,x_k) = \tuple{f(x_1[1],\ldots,x_k[1]),f(x_1[2],\ldots,x_k[2])}$.
Clearly, we can define the set of relations which have a particular
$k$-ary polymorphism $f$ as the set of relations forbidding a
particular set of patterns, namely those which allow $k$ tuples
$x_1,\ldots,x_k$ but which disallow $f(x_1,\ldots,x_k)$. Thus, every
class of binary CSPs defined by having a particular polymorphism can
be defined using forbidden patterns.

We now turn our attention to structural tractability. In
Example~\ref{ex:tw} below we show that a forbidden pattern can
capture the class of CSPs with tree width $1$.

\begin{example}
\label{ex:tw}

\begin{pattern}\vspace{0.3cm}
\centering
\begin{tikzpicture}
 \path[draw,rounded corners=15pt] (0,4) rectangle (1,2.5);
  \fill (0.5,3.25) circle (0.1);

  \path[draw,rounded corners=15pt] (1.5,1.5) rectangle (2.5,0);
  \fill (2,0.75) circle (0.1);

  \path[draw,rounded corners=15pt] (3,3) rectangle (4,1.5);
  \fill (3.5,1.9) circle (0.1);
  \fill (3.5,2.6) circle (0.1);

  \node at (0.5,2.2) {$\var{1}$};
  \node at (2.7,0) {$\var{2}$};
  \node at (3.5,1.2) {$\var{3}$};

  \draw[dashed] (0.5,3.25) -- (3.5,2.6);
  \draw[dashed] (2,0.75) -- (3.5,1.9);
\end{tikzpicture}

$\var{1} < \var{2} < \var{3}$ \caption{Tree structure pattern
(\tree)} \label{pat:tree}
\end{pattern}

Consider the pattern \tree, given as
Pattern~\ref{pat:tree}. We will show that the class \cspneg{\tree} is
exactly the set of CSPs whose constraint graph is a forest (i.e.~has
tree width $1$). First, suppose $P \in \cspneg{\tree}$. Then, there
exists some ordering $\pi=(v_1,\ldots,v_n)$ such that each variable
shares a constraint with at most one variable preceding it in the
ordering. %
On the other hand, suppose $P$ is a CSP whose constraint graph is a
tree. By ordering the vertices according to a pre-order traversal,
we obtain an ordering in which  each variable shares a constraint
with at most one variable preceding it in the ordering (its parent);
thus, $P \in \cspneg{\tree}$.  

\end{example}

\section{On characterising tractable forbidden patterns}
\label{sec:nomaximal}

In relational tractability we can define a maximal tractable
sub-problem of the CSP problem.  Such a class of relations is maximal
in the sense that it is not possible to add even one more relation to
the set without sacrificing tractability.

In the case of structural tractability the picture is less clear,
since here we measure the complexity of an infinite set of
hypergraphs (or, more generally, relational structures).  We obtain
tractability if we have a bound on some width measure of these
structures.  Whatever width measure is chosen we have a containment of
the class with width bounded by $k$ inside that of the class of width
bounded by $k+1$ and so no maximal class is possible. Nevertheless,
for each $k$ there is a unique maximal class of structurally
tractable instances.

In this section, we show that in the case of forbidden patterns the
situation is similar.

\begin{definition}
Let $\chi = \tuple{V,D,C}$ and $\tau=\tuple{V',D',C'}$ be any two
\unstructured\ CSP patterns.  We can form the disjoint unions $V
\dotcup V'$ and $D \dotcup D'$.  Now, extend each constraint pattern
in $C$ to be over the domain $D \dotcup D'$ by setting the value of
any tuple including elements of $D'$ to be $\unknown$, and extend
similarly the constraint patterns in $C'$. In this way we can define
$C \dotcup C'$ and then we set the \defterm{disjoint union} of $\chi$
and $\tau$ to be $\chi \dotcup \tau = \tuple{V\dotcup V', D\dotcup
D', C
  \dotcup C'}$.
\end{definition}

\newcommand{\pat}{\ensuremath{\tau}}

\begin{lemma}\label{lem:nomaximal}
Let $\chi$ and $\tau$ be \unstructured\ binary CSP patterns. Then
 \[\cspneg{\chi} \cup \cspneg{\pat} \subsetneq \cspneg{\chi
  \dotcup \pat}\,.\]
Moreover, we have that \cspneg{\chi\dotcup \pat} is
tractable whenever \cspneg{\chi} and \cspneg{\pat} are tractable.
\end{lemma}
\begin{proof}
We begin by showing the strict inclusion
\[
\cspneg{\chi} \cup \cspneg{\pat} \subsetneq \cspneg{\chi \dotcup
  \pat}\,.
\]
That the inclusion holds follows directly from
Lemma~\ref{lem:embed}. To see that the inclusion is strict,
observe that $\chi$ and \pat\ occur in a CSP pattern
whose domain is the disjoint union of those for $\chi$ and
\pat\, but whose variable set has size equal to the larger of
the two original variable sets. Any CSP instance containing this
pattern is neither in $\cspneg{\chi}$ nor in $\cspneg{\pat}$.
However, we can construct a CSP instance containing this pattern which
is contained in \cspneg{\chi \dotcup \pat}, as the assumption that
all variables are distinct means that $\chi \dotcup \pat$ is not
contained in this pattern.

Suppose $P \in$ \cspneg{\chi \dotcup \pat}. If $P \in \cspneg{\chi}
\cup \cspneg{\pat}$ then $P$ can be solved in polynomial time, by the
tractability of \cspneg{\chi} and \cspneg{\pat}.

So we may suppose that $\chi \occursin P$. Choose a particular
occurrence of $\chi$ in $P$ and let $\sigma$ denote the set of
variables used in the containment. Consider any assignment $t : \sigma
\mapsto D$. Let $P_t$ denote the problem obtained by making this
assignment and then enforcing arc-consistency on the resulting
problem. This corresponds to adding some new unary
constraints to $P$.

We will show that if \pat\ occurs in $P_t$ then $\chi \dotcup
\pat$ must occur in $P$. To see this, observe that any containment of
\pat\ in $P_m$ naturally induces a containment of \pat\ in $P$ that
extends to a containment of $\chi \dotcup \pat$ in $P$, by
considering the occurrence of $\chi$ in $\sigma$. Thus, we can
conclude that $P_t \in \cspneg{\pat}$, and so can be solved in
polynomial time.

By construction, any solution to $P_t$ extends to a solution to $P$ by
adding the assignment $t$ to the variables $\sigma$. Moreover, every
solution to $P$ corresponds to a solution to $P_t$ for some $t :
\sigma \mapsto D$. Since the size of $\chi$ is fixed, we can
iterate over the solutions to $\chi$ in polynomial time. If $P$ has a
solution, then we will find it as the solution to some $P_t$. If we
find that no $P_t$ has a solution, then we known $P$ does not have a
solution. Thus, since we can solve each $P_t$ in polynomial time, we
can also solve $P$ in polynomial time.
\end{proof}
\begin{corollary}
No tractable class defined by forbidding a flat pattern is
maximal.
\end{corollary}
\begin{proof}
  Let $\chi$ be any tractable flat pattern. Consider the pattern
  defined by the disjoint union of two copies of $\chi$, which we
  denote $\chi^{(2)}$. By Lemma~\ref{lem:nomaximal} we have that
  $\cspneg{\chi^{(2)}}$ is tractable but also that
  \[
  \cspneg{\chi} \subsetneq \cspneg{\chi^{(2)}}\,,
  \]
  and hence $\cspneg{\chi}$ is not a maximal tractable class.
\end{proof}

It follows that we cannot characterise tractable forbidden patterns
by exhibiting all maximal classes defined by tractable forbidden
patterns. Indeed, a consequence of Lemma~\ref{lem:nomaximal} is that
we can construct an infinite chain of patterns, such that forbidding
each one gives rise to a slightly larger tractable class. Naturally,
if we place an upper bound on the size of the patterns then
there are only finitely many patterns that we can consider,
so maximal tractable classes defined by forbidden patterns of bounded
size necessarily exist. 

For the moment, we are not able to make a conjecture as to the
structure of a dichotomy for general forbidden patterns. Nonetheless,
in Section \ref{sec:pivot}, by restricting our attention to a special
case, {\em forbidden negative patterns}, we are able to obtain
interesting general results.

\section{Tractable forbidden patterns on three variables}
\label{sec:examples}
In the previous section, we showed that we need to place restrictions
on the size of the forbidden patterns if we want to establish any
sort of dichotomy. Since forbidden patterns on two variables only
place restrictions on the set of constraint relations that can occur
in an instance, the first interesting hybrid classes occur when we
consider three variables. In this section we present two hybrid
tractable classes of binary CSP instances characterised by forbidden
patterns on three variables.

\begin{pattern}\vspace{0.3cm}
\centering
\begin{tikzpicture}
  \path[draw,rounded corners=15pt] (0,4) rectangle (1,2.5);
  \fill (0.5,3.25) circle (0.1);

  \path[draw,rounded corners=15pt] (1.5,1.5) rectangle (2.5,0);
  \fill (2,0.75) circle (0.1);

  \path[draw,rounded corners=15pt] (3,3) rectangle (4,1.5);
  \fill (3.5,1.9) circle (0.1);
  \fill (3.5,2.6) circle (0.1);

  \node at (0.5,2.2) {$\var{1}$};
  \node at (2.7,0) {$\var{2}$};
  \node at (3.5,1.2) {$\var{3}$};
  \node at (4.25,1.9) {$a$};
  \node at (4.25,2.6) {$b$};

  \draw (0.5,3.25) -- (3.5,1.9);
  \draw (2,0.75) -- (3.5,2.6);
  \draw (0.5,3.25) -- (2,0.75);
  \draw[dashed] (0.5,3.25) -- (3.5,2.6);
  \draw[dashed] (2,0.75) -- (3.5,1.9);
\end{tikzpicture}

$\var{1} < \var{2} < \var{3}$ \caption{Broken triangle pattern
(\btp)} \label{pat:btp}
\end{pattern}

The first example, already introduced in \cite{Cooper2010:generalizing}, is
known as the \textbf{broken-triangle property} (\btp). In order to
capture this class by a forbidden pattern we have to work in a context
where the set of variables is totally ordered. In this case
pattern containment must preserve the total order. We can define
the broken-triangle property by the forbidden pattern \btp, shown in
Pattern~\ref{pat:btp}. The following result was proved in
\cite{Cooper2010:generalizing}. 

\begin{theorem}\label{thm:btp}
Let \btp\ be the pattern in Pattern~\ref{pat:btp}. The class of CSP
instances $\cspneg{\btp}$ can be solved in polynomial time.
\end{theorem}

It is easy to see that \tree\ (shown in Pattern~\ref{pat:tree}) occurs
in \btp\ (with some truth-values $U$ being changed to
$T$). It follows from Lemma~\ref{lem:embed} that $\cspneg{\tree}
\subseteq \cspneg{\btp}$. Hence the class $\cspneg{\btp}$ includes
all CSP instances whose constraint graph is a tree. However,
$\cspneg{\btp}$ also includes CSP instances whose constraint graph
has tree width $r$ for any value of $r$: consider, for example, a CSP
instance with $r+1$ variables and an identical constraint between
every pair of variables which simply disallows the single tuple
$\tuple{0,0}$.

For any tractable forbidden pattern relative to a context with an
order on the variables, we can obtain another tractable class by
considering problems forbidding the pattern in a \unstructured\
context. The class obtained is (possibly) smaller because it is
easier to establish containment of the \unstructured\ pattern. In the particular case of
the broken-triangle property, we obtain a strictly smaller tractable
class by forbidding Pattern~\ref{pat:btp} for all triples of
variables $v_{1},v_{2},v_{3}$ irrespective of their order. We can
easily exhibit a CSP instance that shows this inclusion to be strict:
for example, the 3-variable CSP instance over Boolean domains
consisting of the two constraints $\var{1}=\var{2}$,
$\var{1}=\var{3}$ with the variable ordering $\var{1} < \var{2} <
\var{3}$.

Our second example is a generalisation of the well-known tractable
class of problems,
\alldiff+\unary~\cite{Regin1994:filtering,vanHoeve2001:alldifferent}: an
instance of this class consists of a set of variables $V$, a set of
arbitrary unary constraints on $V$, and the constraint $v \neq w$
defined on each pair of distinct variables $v,w \in V$. We define a
more general class containing every such instance using the forbidden
pattern shown in Pattern~\ref{pat:negtrans}, which we call \negtrans.
Forbidding this pattern insists that disallowing tuples is a
transitive relation, i.e.~if $(\tuple{v,a},\tuple{x,b})$
and $(\tuple{x,b},\tuple{w,c})$ are disallowed then
$(\tuple{v,a},\tuple{w,c})$ must also be disallowed. By
the transitivity of equality, Pattern~\ref{pat:negtrans} does not
occur in any binary CSP instance in the class \alldiff+\unary.

\begin{pattern}\vspace{0.3cm}
\centering
\begin{tabular}{cc}
\begin{tikzpicture}
  \node at (0.5,4.2) {$v$};
  \path[draw,rounded corners=15pt] (0,4) rectangle (1,2.5);
  \fill (0.5,3.25) circle (0.1);

  \node at (2,1.7) {$w$};
  \path[draw,rounded corners=15pt] (1.5,1.5) rectangle (2.5,0);
  \fill (2,0.75) circle (0.1);

  \node at (3.5,3.2) {$x$};
  \path[draw,rounded corners=15pt] (3,3) rectangle (4,1.5);
  \fill (3.5,2.25) circle (0.1);

  \draw[dashed] (0.5,3.25) -- (3.5,2.25);
  \draw[dashed] (2,0.75) -- (3.5,2.25);
  \draw (2,0.75) -- (0.5,3.25);
\end{tikzpicture}
&\hspace{1cm}
\\
NEQ($v,w,x$)
&
\end{tabular}
\caption{Negative transitive pattern (\negtrans)}
\label{pat:negtrans}
\end{pattern}
\begin{theorem} \label{thm:negtrans}
Let \negtrans\  denote Pattern~\ref{pat:negtrans}. The class of
CSP instances forbidding \negtrans\ can be
solved in polynomial time. 
\end{theorem}
\begin{proof}
We prove this by a straightforward reduction to the well-known
tractable problem \alldiff+\unary~\cite{Costa1994:persistency,Regin1994:filtering}.

Let $P=\tuple{V,D,C}$ be a binary CSP in which \negtrans\ does not
occur, and let $n=|V|$ and $d=|D|$. We define the graph $G_P$ which we
call the \emph{inconsistency graph of P}. The vertices of $G_P$ are
the pairs \tuple{v,c} where $v$ is a variable in $P$ and $c \in D$ is
allowed by the unary constraint on $v$. The edges of $G_P$ are the
pairs of vertices $\{\tuple{v,a},\tuple{w,b}\}$ of $G_P$ for which
there exists a constraint $\tuple{\tuple{v,w},\rho}$ with scope
$\tuple{v,w}$ such that $\tuple{a,b} \not \in \rho$. (The 
inconsistency graph is the microstructure
complement~\cite{Jegou1993:decomposition} without edges between pairs
of assignments to the same variable.)

We first prove that, for any connected component  $H$ of $G_P$, either
\begin{itemize}
\item The subgraph of $G_P$ induced by $H$ is a complete
  multipartite graph with edges $\{\tuple{v,a},\tuple{w,b}\}$ for each
  $\tuple{v,a},\tuple{w,b} \in H$ satisfying $v \neq w$ (in this
  case, we call H an \emph{inconsistency clique}), or
\item $H$ meets exactly two variables of $P$: $|\{v \mid \tuple{v,a}
  \in H\}| = 2$.
\end{itemize}

Any connected component $H$ that meets only one variable is a trivial
inconsistency clique. Consider a component $H$ that meets at least
three distinct variables. To show that $H$ is an inconsistency clique
we have only to show that the two end-points of every path of length
three meeting three variables and of every path of length four
meeting two variables are connected. (The length of a path is the
number of vertices on the path).

Let $\tuple{v,a},\tuple{x,c},\tuple{w,b}$, be a path of length three
in $H$, where $v,x,w$ are distinct variables. Since \negtrans\ does
not occur in $P$, we must have a constraint
$\tuple{\tuple{v,w},\rho}$ with $\tuple{a,b} \not \in \rho$.

Let $\tuple{v,a},\tuple{w,a'},\tuple{v,b'},\tuple{w,b}$ be a path of
length four in $H$.  Since $H$ is connected and comprises at least
three variables, there is some other vertex \tuple{x,c} with $x
\not\in \{v,w\}$ which is connected to this path. Since this creates
paths of length three through three variables we can repeatedly use
the argument given above to show that \tuple{x,c} is adjacent to each
of the four vertices on the path. Finally, since it is adjacent to
both \tuple{v,a} and \tuple{w,b} we use the argument one last time to
show that these two vertices are adjacent.

We can now demonstrate the reduction to \alldiff+\unary.  First we can
identify all connected components of $G_P$ in polynomial time.  For
each component $H$ that is not an inconsistency clique, $H$ meets
exactly two variables $v,w$ and there is some pair of vertices
\tuple{v,a} and \tuple{w,b} which are not adjacent and which are
adjacent to no vertex $\tuple{x,c}$ for any other variable $x$. We
can therefore make the assignments $v=a,w=b$ and remove from $G_P$ all vertices
corresponding to assignments to these variables. We denote
by $V'$ the remaining set of variables after removing each such pair
of variables from $P$. Note that we have an assignment to every
variable not in $V'$ that is consistent with any assignment to the
variables of $V'$.

Let $P_{V'}$ denote the resulting CSP instance on variables $V'$ and
$G_{P_{V'}}$ the corresponding inconsistency graph. The components
$H_{1},\ldots,H_{m}$ of $G_{P_{V'}}$ are all inconsistency cliques.
For each component $H_{i}$ and each variable $v$ we define $H_{i}(v)
= \{\tuple{w,c} \in H_{i} \mid w = v\}$.

Consider a CSP $P'$ with variables $V'$ and domain $\{1,\ldots,m\}$.
Apply the unary constraint on each variable $v$ of $P'$ given by the
unary relation $\{\tuple{i} \mid H_i(v) \neq \emptyset\}$. Finally
apply the \alldiff\ constraint over all variables of $P'$.

No solution to $P_{V'}$ can contain two assignments from the same
component of $G_{P_{V'}}$. Therefore, to every solution $s$ to
$P_{V'}$ there is a corresponding solution $s'$ to $P'$: choose $s'(v)
= i$ where $\tuple{v,s(v)} \in H_i$.

Conversely, any solution $s'$ to $P'$ corresponds to a solution
$s$ to $P_{V}$ by choosing $s(v)$ to be any value in $H_{s'(v)}(v)$,
for each $v \in V$.

The time taken to obtain $G_{P_{V'}}$ from $P$ is proportional to the
total number of disallowed tuples in $P$; hence, the time taken is
$O(|C|d^2)$. Solving $P'$ is equivalent to finding a perfect matching
in a bipartite graph with $|V'| + m$ vertices and up to $|V'|m$
edges. Using the {\em Fibonacci heap} data structure, we can find a
perfect matching in a bipartite graph with $N$ vertices and $M$ edges
in time $O(N^2\log(N) + NM)$~\cite{Fredman1987:fibonacci}. Thus, we can
find a solution to $P'$ in time $O((n+m)^2\log(n+m) + (n+m)nm)\,.$ The
maximum value of $m$ occurs when each component of $G_{P}$ contains
exactly three assignments, so we will always have $m \leq
\frac{nd}{3}$. Thus, under the reasonable assumption that $d \leq n$,
we can solve $P'$ in time $O(n^3d^2)$. Since $|C|$ is $O(n^2)$, it
follows that $P$ can be solved in time $O(n^3d^2)$.    
\end{proof}

It has recently been shown \cite{Cooper2010:hybrid} that the tractable
class defined by forbidding Pattern~\ref{pat:negtrans} (\negtrans)
can be extended to soft constraint problems but that this is not the
case for the class of problems obtained by forbidding
Pattern~\ref{pat:btp} (\btp) (in the sense that the class becomes
NP-hard if all unary soft constraints are also allowed).

Having demonstrated through the \btp\ and \negtrans\ patterns that
forbidding patterns provides a language which is rich enough to
define interesting hybrid tractable classes, we concentrate in the
rest of the paper on progress towards characterising tractable
forbidden patterns.

\section{Binary \unstructured\ negative patterns}
\label{sec:pivot}

In this section we define three particular patterns and one
infinite class of patterns.  We then use these patterns to
characterise a very large class of intractable patterns.  We
prove that any finite set of patterns not in this class has a simple
structure: one of the patterns must contain one of a particular set of
patterns, which we call {\em pivots}. This means that any tractable
finite set of patterns must include a pattern in which a pivot pattern
occurs.  Furthermore, we exhibit a class of patterns which are contained in
pivots and which we are able to prove give rise to a tractable
class. We conjecture that pivots are also tractable; if this is the
case then it implies a simple characterisation of the tractability
of finite sets of binary \unstructured\ {\em negative} patterns.

\begin{definition}
A constraint pattern \tuple{\sigma,\rho} will be called
\textbf{negative} if $\rho$ never takes the value $T$.  A CSP pattern $\chi$
is negative if every constraint pattern in $\chi$ is negative.
\end{definition}

\begin{pattern}\vspace{0.3cm}
\centering
\begin{tikzpicture}
  \draw[rounded corners=15pt] (0,3.5) rectangle (1,2);
  \fill (0.5,2.5) circle (0.1);
  \fill (0.5,3) circle (0.1);
  \node at (0.5,1.7) {$\var{1}$};
  \node at (-0.5,2.5) {$c$};
  \node at (-0.5,3) {$c'$};

  \draw[rounded corners=15pt] (2.5,3.5) rectangle (3.5,2);
  \fill (3,2.5) circle (0.1);
  \fill (3,3) circle (0.1);
  \node at (3,1.7) {$\var{2}$};

  \draw[rounded corners=15pt] (5,3.5) rectangle (6,2);
  \fill (5.5,2.5) circle (0.1);
  \fill (5.5,3) circle (0.1);
  \node at (5.5,1.7) {$\var{3}$};

  \draw[rounded corners=15pt] (1.5,1.5) rectangle (2.5,0);
  \fill (2,0.5) circle (0.1);
  \fill (2,1) circle (0.1);
  \node at (2,-0.3) {$\var6$};

  \draw[rounded corners=15pt] (4,1.5) rectangle (5,0);
  \fill (4.5,0.5) circle (0.1);
  \fill (4.5,1) circle (0.1);
  \node at (4.5,-0.3) {$\var5$};

  \draw[rounded corners=15pt] (6.5,1.5) rectangle (7.5,0);
  \fill (7,0.5) circle (0.1);
  \fill (7,1) circle (0.1);
  \node at (7,-0.3) {$\var4$};

  \draw[dashed] (0.5,3) -- (3,2.5);
  \draw[dashed] (3,3) -- (5.5,2.5);
  \draw[dashed] (5.5,3) -- (7,1);
  \draw[dashed] (7,0.5) -- (4.5,1);
  \draw[dashed] (4.5,0.5) -- (2,1);
  \draw[dashed] (2,0.5) -- (0.5,2.5);
\end{tikzpicture}

Context: $NEQ(\var{1},\ldots,\var{6})$
\caption{\Cycle(6)}
\label{pat:cycle}
\end{pattern}

\begin{pattern}\vspace{0.3cm}
\centering
\begin{tikzpicture}
  \node at (0,1) {$\varc{3}$};
  \draw (1,1) circle (0.5);
  \fill (1,1) circle (0.1);

  \node at (0,3) {$\varc{2}$};
  \draw (1,3) circle (0.5);
  \fill (1,3) circle (0.1);

  \node at (0,5) {$\varc{1}$};
  \draw (1,5) circle (0.5);
  \fill (1,5) circle (0.1);

  \draw[rounded corners=15pt] (2.5,4.5) rectangle (3.5,1.5);
  \fill (3,2) circle (0.1);
  \fill (3,3) circle (0.1);
  \fill (3,4) circle (0.1);

  \draw[dashed] (1,1) -- (3,2);
  \draw[dashed] (1,3) -- (3,3);
  \draw[dashed] (1,5) -- (3,4);

  \node at (8,1) {$\varc{3}'$};
  \draw (7,1) circle (0.5);
  \fill (7,1) circle (0.1);

  \node at (8,3) {$\varc{2}'$};
  \draw (7,3) circle (0.5);
  \fill (7,3) circle (0.1);

  \node at (8,5) {$\varc{1}'$};
  \draw (7,5) circle (0.5);
  \fill (7,5) circle (0.1);

  \draw[rounded corners=15pt] (4.5,4.5) rectangle (5.5,1.5);
  \fill (5,2) circle (0.1);
  \fill (5,3) circle (0.1);
  \fill (5,4) circle (0.1);

  \draw[dashed] (5,2) -- (7,1);
  \draw[dashed] (5,3) -- (7,3);
  \draw[dashed] (5,4) -- (7,5);
\end{tikzpicture}
\vspace{0.5cm}

Context: $NEQ(\varc{1},\varc{2},\varc{3},\varc{1}') \wedge
NEQ(\varc{1}',\varc{2}',\varc{3}')$
\caption{\Valency}
\label{pat:valency}
\end{pattern}

\begin{pattern}\vspace{0.3cm}
\centering
\vspace{0.3cm}
\begin{tikzpicture}
  \draw (0.5,1) circle (0.5);
  \fill (0.5,1) circle (0.1);
  \node at (0.5,0) {$\var{1}$};

  \draw (0.5+1.5,1) circle (0.5);
  \fill (0.5+1.5,1) circle (0.1);
  \node at (0.5+1.5,0) {$\var{2}$};

  \draw (0.5+2*1.5,1) circle (0.5);
  \fill (0.5+2*1.5,1) circle (0.1);
  \node at (0.5+2*1.5,0) {$\var{3}$};

  \draw[dashed] (0.5,1) -- (0.5+1.5,1) -- (0.5+2*1.5,1);

  \draw (0.5+5,1) circle (0.5);
  \fill (0.5+5,1) circle (0.1);
  \node at (0.5+5,0) {$\varb{1}$};

  \draw (0.5+1.5+5,1) circle (0.5);
  \fill (0.5+1.5+5,1) circle (0.1);
  \node at (0.5+1.5+5,0) {$\varb{2}$};

  \draw (0.5+2*1.5+5,1) circle (0.5);
  \fill (0.5+2*1.5+5,1) circle (0.1);
  \node at (0.5+2*1.5+5,0) {$\varb{3}$};

  \draw[dashed] (0.5+5,1) -- (0.5+1.5+5,1) -- (0.5+2*1.5+5,1);
\end{tikzpicture}
\vspace{0.5cm}\\
$NEQ(\var{1},\var{2},\var{3},\varb{1}) \wedge
NEQ(\varb{1},\varb{2},\varb{3})$
\caption{\Path} \label{pat:path}
\end{pattern}

\begin{pattern}
\centering
\begin{tikzpicture}
 \node at (0,1) {$\var{3}$};
  \draw (1,1) circle (0.5);
  \fill (1,1) circle (0.1);

  \node at (0,3) {$\var{2}$};
  \draw (1,3) circle (0.5);
  \fill (1,3) circle (0.1);

  \node at (0,5) {$\var{1}$};
  \draw (1,5) circle (0.5);
  \fill (1,5) circle (0.1);

  \draw[rounded corners=15pt] (2.5,4.5) rectangle (3.5,1.5);
  \fill (3,2) circle (0.1);
  \fill (3,3) circle (0.1);
  \fill (3,4) circle (0.1);
  \node at (3,1.2) {$x$};

  \draw[dashed] (1,1) -- (3,2);
  \draw[dashed] (1,3) -- (3,3);
  \draw[dashed] (1,5) -- (3,4);

  \draw (0.5+5,1+2) circle (0.5);
  \fill (0.5+5,1+2) circle (0.1);
  \node at (0.5+5,0+2) {$\varb{1}$};

  \draw (0.5+1.5+5,1+2) circle (0.5);
  \fill (0.5+1.5+5,1+2) circle (0.1);
  \node at (0.5+1.5+5,0+2) {$\varb{2}$};

  \draw (0.5+2*1.5+5,1+2) circle (0.5);
  \fill (0.5+2*1.5+5,1+2) circle (0.1);
  \node at (0.5+2*1.5+5,0+2) {$\varb{3}$};

  \draw[dashed] (0.5+5,1+2) -- (0.5+1.5+5,1+2) -- (0.5+2*1.5+5,1+2);
\end{tikzpicture}

$NEQ(\var{1},\var{2},\var{3})$, $NEQ(\varb{1},\varb{2},\varb{3})$, and
$x \neq \varb{2}$

\caption{\Valency+\Path}
\label{pat:valpath}
\end{pattern}

In Definition~\ref{defn:conpat} below, we define the concept of a
connected negative binary pattern. 
These correspond to negative binary patterns $\chi$ such that every realisation
of $\chi$ as a binary CSP instance has a connected constraint graph. We
first generalise the notion of constraint graph to CSP patterns. We
call the resulting graph the negative structure graph.
\begin{definition}\label{defn:conpat}
Let $\chi$ be any binary negative
pattern.  The vertices of the \defterm{negative structure graph} $G$
are the variables of $\chi$.  A pair of vertices is connected in $G$
if and only if they form a scope in $\chi$ whose constraint pattern
assigns at least one tuple the value $F$. We say that $\chi$ is
\defterm{connected} if its negative structure graph is connected.
\end{definition}

For example, Pattern~\ref{pat:valency} (\Valency),
Pattern~\ref{pat:path} (\Path) and Pattern~\ref{pat:valpath}
(\Valency+\Path) are not connected. Note that a pattern which is not
connected may occur in a connected pattern (and vice versa).
Pattern~\ref{pat:cycle} shows \Cycle(6) which is connected. This is just one example of
the generic pattern \Cycle($k$) where $k \geq 2$. The only context
for \Cycle($k$) is that all variables are distinct, except for the
special case $k=2$ for which the context also includes $c \neq c'$,
meaning that \Cycle(2) is composed of a single binary constraint
pattern containing two {\em distinct} inconsistent tuples. The
following theorem uses these patterns to show that most patterns are
intractable.

\begin{theorem}
\label{thm:intractable}
Let $\mathcal X$ be any finite set of connected negative binary
patterns. If at least one of $\Cycle(k)$ (for some $k \geq 2$),
\Valency, \Path, or \Valency$+$\Path\ occurs in each $\chi \in \mathcal
X$ then \cspneg{\mathcal X} is intractable.
\end{theorem}
\begin{proof}
Let $\mathcal X$ be a set of connected negative binary patterns and let $\ell$ be
the number of variables in the largest member of $\mathcal X$.

Assuming at least one of the four patterns occurs in each $\chi \in
\mathcal X$, we can construct a class of CSPs in which no element
of $\mathcal X$ occurs and to which we have a polynomial reduction from
the well-known NP-complete problem
3SAT~\cite{Garey1979:computers}.

The construction will involve three gadgets, examples of which are
shown in Figure~\ref{fig3SAT}. These gadgets each serve a particular
purpose:
\begin{enumerate}%
\item  The {\em cycle gadget}, given in Figure~\ref{fig3SAT}(a) for the special
  case of $4$ variables, enforces that a cycle of Boolean variables
  $(\var{1},\var{2},\ldots,\var{r})$ all take the same value.
\item The {\em clause gadget} in Figure~\ref{fig3SAT}(b) is equivalent to the clause
  $\var{1} \vee \var{2} \vee \var{3}$, since $\var{C}$ has a value in
  its domain if and only if one of the three $\var{i}$ variables is
  set to true. We can obtain all other 3-clauses on these three
  variables by inverting the domains of the $\var{i}$ variables.
\item The {\em line gadget} in Figure~\ref{fig3SAT}(c), imposes the constraint $\var{1}
  \Rightarrow \var{2}$. It can also be used to impose the logically
  equivalent constraint $\neg \var{2} \Rightarrow \neg \var{1}$.
\end{enumerate}

Now, suppose that we have an instance of 3SAT with $n$ propositional
variables $X_1,\ldots,X_n$ and $m$ clauses $C_1,\ldots,C_m$.

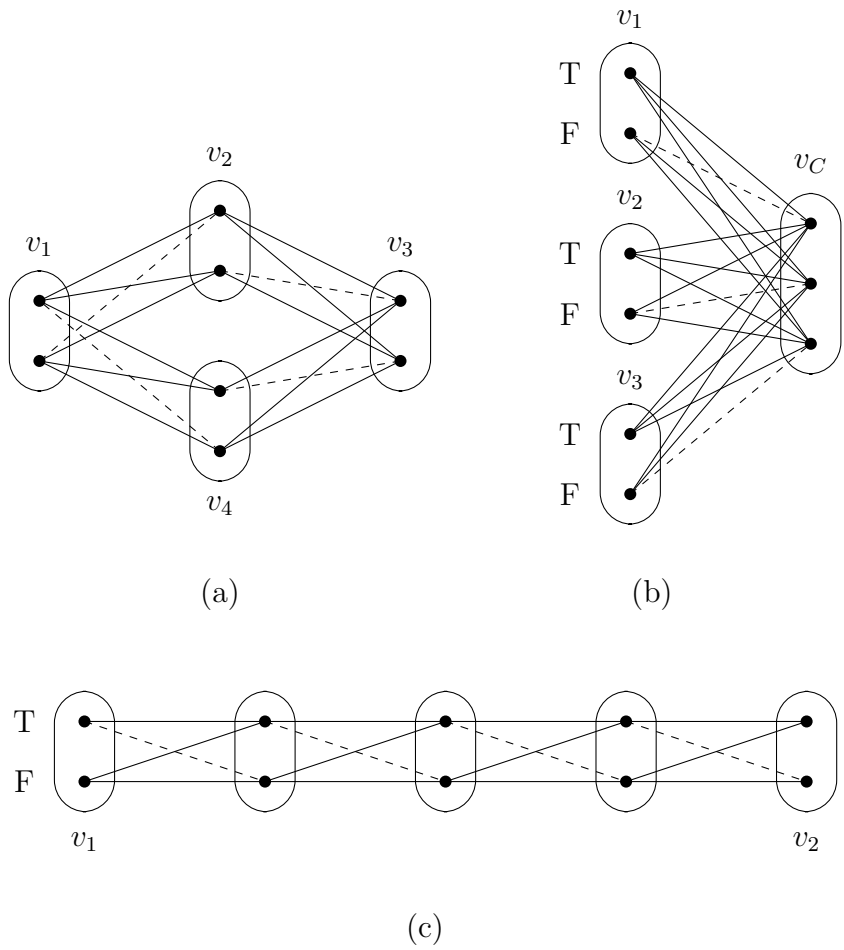
\begin{figure}
\centering
\begin{tabular}{cc}
  \begin{tikzpicture}[scale=0.8]
    \draw[rounded corners=12pt] (0,4.5) rectangle (1,2.5);
    \node at (0.5,4.9) {$\var{1}$};
    \fill (0.5,3) circle (0.1);
    \fill (0.5,4) circle (0.1);

    \draw[rounded corners=12pt] (3,6) rectangle (4,4);
    \node at (3.5,6.4) {$\var{2}$};
    \fill (3.5,4.5) circle (0.1);
    \fill (3.5,5.5) circle (0.1);

    \draw[rounded corners=12pt] (6,4.5) rectangle (7,2.5);
    \node at (6.5,4.9) {$\var{3}$};
    \fill (6.5,3) circle (0.1);
    \fill (6.5,4) circle (0.1);

    \draw[rounded corners=12pt] (3,3) rectangle (4,1);
    \node at (3.5,0.6) {$\var4$};
    \fill (3.5,1.5) circle (0.1);
    \fill (3.5,2.5) circle (0.1);

    \draw[dashed] (0.5,3) -- (3.5,5.5);
    \draw[dashed] (3.5,4.5) -- (6.5,4);
    \draw[dashed] (6.5,3) -- (3.5,2.5);
    \draw[dashed] (3.5,1.5) -- (0.5,4);
    \draw (0.5,3) -- (3.5,1.5);
    \draw (0.5,3) -- (3.5,2.5);
    \draw (0.5,4) -- (3.5,2.5);
    \draw (0.5,3) -- (3.5,4.5);
    \draw (0.5,4) -- (3.5,4.5);
    \draw (0.5,4) -- (3.5,5.5);
    \draw (6.5,3) -- (3.5,1.5);
    \draw (6.5,4) -- (3.5,1.5);
    \draw (6.5,4) -- (3.5,2.5);
    \draw (6.5,3) -- (3.5,4.5);
    \draw (6.5,3) -- (3.5,5.5);
    \draw (6.5,4) -- (3.5,5.5);
  \end{tikzpicture}

  & \hspace{1cm}

  \begin{tikzpicture}[scale=0.8]
    \node at (1,8.4) {$\var{1}$};
    \draw[rounded corners=12pt] (0.5,8) rectangle (1.5,6);
    \fill (1,7.5) circle (0.1);
    \node at (0,7.5) {T};
    \fill (1,6.5) circle (0.1);
    \node at (0,6.5) {F};

    \node at (1,5.4) {$\var{2}$};
    \draw[rounded corners=12pt] (0.5,5) rectangle (1.5,3);
    \fill (1,4.5) circle (0.1);
    \node at (0,4.5) {T};
    \fill (1,3.5) circle (0.1);
    \node at (0,3.5) {F};

    \node at (1,2.4) {$\var{3}$};
    \draw[rounded corners=12pt] (0.5,2) rectangle (1.5,0);
    \fill (1,1.5) circle (0.1);
    \node at (0,1.5) {T};
    \fill (1,0.5) circle (0.1);
    \node at (0,0.5) {F};

    \node at (4,6) {$\var{C}$};
    \draw[rounded corners=12pt] (3.5,5.5) rectangle (4.5,2.5);
    \fill (4,5) circle (0.1);
    \fill (4,4) circle (0.1);
    \fill (4,3) circle (0.1);

    \draw[dashed] (1,6.5) -- (4,5);
    \draw[dashed] (1,3.5) -- (4,4);
    \draw[dashed] (1,0.5) -- (4,3);
    \draw (1,0.5) -- (4,4);
    \draw (1,0.5) -- (4,5);
    \draw (1,1.5) -- (4,3);
    \draw (1,1.5) -- (4,4);
    \draw (1,1.5) -- (4,5);
    \draw (1,3.5) -- (4,3);
    \draw (1,3.5) -- (4,5);
    \draw (1,4.5) -- (4,3);
    \draw (1,4.5) -- (4,4);
    \draw (1,4.5) -- (4,5);
    \draw (1,6.5) -- (4,3);
    \draw (1,6.5) -- (4,4);
    \draw (1,7.5) -- (4,3);
    \draw (1,7.5) -- (4,4);
    \draw (1,7.5) -- (4,5);
  \end{tikzpicture}

  \\\\
  (a) & (b)
\end{tabular}

\vspace{1cm}

\begin{tikzpicture}[scale=0.8]
  \node at (0,1) {F};
  \node at (0,2) {T};

  \draw[rounded corners=12pt] (0.5,2.5) rectangle (1.5,0.5);
  \fill (1,1) circle (0.1);
  \fill (1,2) circle (0.1);
  \node at (1,0) {$\var{1}$};

  \draw[rounded corners=12pt] (0.5+3,2.5) rectangle (1.5+3,0.5);
  \fill (1+3,1) circle (0.1);
  \fill (1+3,2) circle (0.1);

  \draw[rounded corners=12pt] (0.5+6,2.5) rectangle (1.5+6,0.5);
  \fill (1+6,1) circle (0.1);
  \fill (1+6,2) circle (0.1);

  \draw[rounded corners=12pt] (0.5+9,2.5) rectangle (1.5+9,0.5);
  \fill (1+9,1) circle (0.1);
  \fill (1+9,2) circle (0.1);

  \draw[rounded corners=12pt] (0.5+12,2.5) rectangle (1.5+12,0.5);
  \fill (1+12,1) circle (0.1);
  \fill (1+12,2) circle (0.1);
  \node at (1+12,0) {$\var{2}$};
   \draw (1,1) -- (4,2);
   \draw (4,1) -- (7,2);
   \draw (7,1) -- (10,2);
   \draw (10,1) -- (13,2);
   \draw (1,1) -- (4,1);
   \draw[dashed] (1,2) -- (4,1);
   \draw (1,2) -- (4,2);
   \draw (4,1) -- (7,1);
   \draw[dashed] (4,2) -- (7,1);
   \draw (4,2) -- (7,2);
   \draw (7,1) -- (10,1);
   \draw[dashed] (7,2) -- (10,1);
   \draw (7,2) -- (10,2);
   \draw (10,1) -- (13,1);
   \draw[dashed] (10,2) -- (13,1);
   \draw (10,2) -- (13,2);
\end{tikzpicture}
\vspace{5mm}
\\
(c)
\caption{(a) Making copies of the same variable
  ($\var{1}=\var{2}=\var{3}=\var{4}$). (b) Imposing the ternary
  constraint $\var{C} = \var{1} \vee \var{2} \vee \var{3}$. (c) A line
  of constraints which imposes $\var{1} \Rightarrow \var{2}$.}
\label{fig3SAT}
\end{figure}
We begin to build a CSP instance $\Psi$ to solve this 3SAT instance by using
$n$ copies of the cycle gadget (Figure \ref{fig3SAT}(a)), each
with $m(\ell+1)$ variables. For $i=1,\ldots,n$, the variables along the
$i$th copy of this cycle are denoted by
$(\var{i}^1,\var{i}^2,\ldots,\var{i}^{m(\ell+1)})$.
In any solution to a CSP instance $P_\Psi$ with these and other
constraints, we will have that the variables $\var{i}^{j},
j=1,\ldots,m(\ell+1)$ must all have the same value, $d_i$. We can therefore 
consider each $\var{i}^j$ as a copy of $X_i$.

Consider the clause $C_w$.  There are eight cases to consider but
they are all very similar so we will show the details for just one
case. Suppose that $C_w \equiv X_i \vee X_j \vee \neg X_k$. We build the
clause gadget (Figure \ref{fig3SAT}(b)) with the three Boolean
variables being $c_w^i, c_w^j$ and $c_w^k$ and invert the domain of
$c_w^k$ since it occurs negatively in $C_w$.  Then any solution $s$ to
our constructed CSP must satisfy $s(c_w^i) \vee s(c_w^j) \vee \neg
s(c_w^k) = T$.

We complete the insertion of $C_w$ into the CSP instance by adding
some length $\ell+1$ line constructions (Figure \ref{fig3SAT}(c)).  We
connect the cycle gadgets corresponding to
$X_{i}$, $X_{j}$ and $X_{k}$ to the clause gadget for
clause $C_w$ since $X_{i},X_{j}$ and $X_{k}$ occur in $C_w$.  We
connect $\var{i}^{w(\ell+1)}$ to $c_w^i$ since $X_i$ is positive in
$C_w$, so $s(c_w^i)=T$ is only possible when $s(\var{i}^{w(\ell+1)})=T$,
for any solution $s$.
Similarly, we connect $\var{j}^{w(\ell+1)}$ to $c^j_w$. Finally, since
$X_k$ occurs negatively in $C_w$, we impose the line constraints in
the other direction. This ensures that $s(c_w^k)=F$ is only possible
when $s(\var{k}^{w(\ell+1)})=F$. Imposing these constraints ensures
that a solution is only possible when at least one of the cycles
corresponding to variables $X_i$, $X_j$, and $X_k$ is assigned a value
that would make the corresponding literal in $C_w$ true.

We continue this construction for each clause of the 3SAT instance.
Since $\ell$ is a constant, this is clearly a polynomial reduction from
3SAT.

We now show that any CSP instance $P_\Psi$ constructed in the manner we have
just described cannot contain any pattern in $\mathcal X$. We do this by
showing that no pattern containing \Cycle(k) (for $2 \leq k \leq \ell$), \Valency, \Path, or
\Valency+\Path\ can occur in the instance. This
is sufficient to show that the CSP $P_\Psi$ does not contain any
of the patterns in $\mathcal X$.

In the CSP $P_\Psi$ no constraint contains more than one inconsistent
tuple. Thus, any $\chi \in \mathcal X$ for which $\Cycle(2) \occursin
\chi$ cannot occur in $P_\Psi$. Furthermore, $P_\Psi$ is built from
cycles of length $m(\ell+1)$ and paths of length $\ell+1$, and so cannot
contain any cycles on less than $\ell+1$ vertices. Thus, since $\ell$ is
the maximum number of vertices in any component of $\mathcal X$, it
follows that no $\chi \in \mathcal X$ for which $\Cycle(k) \occursin
\chi$, for any $k \geq 3$, can occur in $P_\Psi$.

We define the \emph{valency} of a variable $x$ to be the number of
distinct variables which share a constraint with $x$. Suppose
$\Valency\ \occursin \chi$. For this to be possible we require that
there is a variable of valency four in $\chi$, or a pair of variables
of valency three connected by a path of length at most $\ell$. Certainly
$P_\Psi$ has no variables of valency four. Moreover, the fact that
$P_\Psi$ was built using paths of length $\ell+1$ means that no two of
its valency three variables are joined by a path of length at most
$\ell$. Thus, any $\chi \in \mathcal X$ for which $\Valency \occursin \chi$
does not occur in $P_\Psi$.

Next, consider that case when $\Path \occursin \chi$. Here
$\chi$ must have two distinct (but possibly overlapping) three-variable lines
(with inconsistent tuples in these constraints that match at domain
values) separated by at most
$\ell$ variables. The only place where inconsistent tuples can meet in $P_\Psi$
is when we connect the line gadget to the cycle gadget. These
connection sites are always at distance greater than $\ell$, so we can
conclude that $\chi \notoccursin P_{\psi}$ whenever $\Path \occursin
\chi$.

Finally, consider the case where $\Valency\mbox{+}\Path \occursin
\chi$. Here, $\chi$ must have one variable of valency at least $3$
and a path of constraints on three variables with intersecting
negative tuples, and these must be connected by a path of less than
$\ell$ variables. As observed above, the only places where we can have
inconsistent tuples meeting is where the line gadget meets the cycle
gadget, and there is a path of at least $\ell$ variables between each
one of these points and every other variable of valency $3$. Thus,
$\chi \notoccursin P_{\psi}$ whenever $\Valency\mbox{+}\Path
\occursin \chi$.
\end{proof}

\begin{pattern}
\vspace{0.3cm}
\centering
\begin{tikzpicture}[scale=0.8]
  \node at (0.5,6.3) {$v_3$};
  \draw[rounded corners=12pt] (0,6) rectangle (1,4);
  \fill (0.5,5.5) circle (0.1);
  \fill (0.5,4.5) circle (0.1);
  \draw[dashed] (0.5,4.5) -- (2.25+0.5,5.5);

  \node at (2.75,6.2) {$v_2$};
  \draw[rounded corners=12pt] (2.25,6) rectangle (2.25+1,4);
  \fill (2.25+0.5,5.5) circle (0.1);
  \fill (2.25+0.5,4.5) circle (0.1);
  \draw[dashed] (2.25+0.5,4.5) -- (2*2.25+0.5,5.5);

  \node at (5,6.2) {$v_1$};
  \draw[rounded corners=12pt] (4.5,6) rectangle (4.5+1,4);
  \fill (2*2.25+0.5,5.5) circle (0.1);
  \fill (2*2.25+0.5,4.5) circle (0.1);
  \draw[dashed] (2*2.25+0.5,4.5) -- (3*2.25+0.5,3.5);

  \node at (0.5,2.2) {$x_3$};
  \draw[rounded corners=12pt] (0,2) rectangle (1,0);
  \fill (0.5,1.5) circle (0.1);
  \fill (0.5,0.5) circle (0.1);
  \draw[dashed] (0.5,0.5) -- (2.25+0.5,1.5);

  \node at (2.75,2.2) {$x_2$};
  \draw[rounded corners=12pt] (2.25,2) rectangle (2.25+1,0);
  \fill (2.25+0.5,1.5) circle (0.1);
  \fill (2.25+0.5,0.5) circle (0.1);
  \draw[dashed] (2.25+0.5,0.5) -- (2*2.25+0.5,1.5);

  \node at (5,2.2) {$x_1$};
  \draw[rounded corners=12pt] (4.5,2) rectangle (4.5+1,0);
  \fill (2*2.25+0.5,1.5) circle (0.1);
  \fill (2*2.25+0.5,0.5) circle (0.1);
  \draw[dashed] (2*2.25+0.5,0.5) -- (3*2.25+0.5,2.5);

  \node at (3*2.25+0.5,4+0.5) {$p$};
  \draw[rounded corners=12pt] (3*2.25,4) rectangle (3*2.25+1,2);
  \fill (3*2.25+0.5,3.5) circle (0.1);
  \node at (3*2.25+0.5-0.7,3.5) {$a$};
  \fill (3*2.25+0.5,2.5) circle (0.1);
  \node at (3*2.25+0.5-0.7,2.5) {$b$};

  \draw[dashed] (3*2.25+0.5,3.5) -- (4*2.25+0.5,2.5+1);

  \node at (4*2.25+0.5,4+1+0.2) {$w_1$};
  \draw[rounded corners=12pt] (4*2.25,4+1) rectangle (4*2.25+1,2+1);
  \fill (4*2.25+0.5,3.5+1) circle (0.1);
  \fill (4*2.25+0.5,2.5+1) circle (0.1);
  \draw[dashed] (4*2.25+0.5,3.5+1) -- (5*2.25+0.5,2.5+1);

  \node at (5*2.25+0.5,4+1+0.2) {$w_2$};
  \draw[rounded corners=12pt] (5*2.25,4+1) rectangle (5*2.25+1,2+1);
  \fill (5*2.25+0.5,3.5+1) circle (0.1);
  \fill (5*2.25+0.5,2.5+1) circle (0.1);
  \draw[dashed] (5*2.25+0.5,3.5+1) -- (6*2.25+0.5,2.5+1);

  \node at (6*2.25+0.5,4+1+0.2) {$w_1$};
  \draw[rounded corners=12pt] (6*2.25,4+1) rectangle (6*2.25+1,2+1);
  \fill (6*2.25+0.5,3.5+1) circle (0.1);
  \fill (6*2.25+0.5,2.5+1) circle (0.1);
\end{tikzpicture}

$$NEQ(p,v_1,v_2,v_3,w_1,w_2,w_3,x_1,x_2,x_3)$$
\caption{\Pivot$(3)$}
\label{pat:pivot}
\end{pattern}

It remains to consider which sets of negative binary patterns could
be tractable. For this, we need to define the {\em pivot} patterns,
\Pivot$(r)$, which contain every tractable pattern.
\begin{definition}
Let $V = \{p\} \cup \{v_1,\ldots,v_r\} \cup \{w_1,\ldots,w_r\} \cup
\{x_1,\ldots,x_r\}$ and let $D = \{a,b\}$. We define the pattern
$\Pivot(r)=(V,D,C_p \cup C_v \cup C_w \cup C_x)$, where
\begin{align*}
  C_p & = \{
  \{(\tuple{p,a},\tuple{v_1,b})\}\mapsto F,
  \{(\tuple{p,a},\tuple{w_1,b})\}\mapsto F,
  \{(\tuple{p,b},\tuple{x_1,b})\}\mapsto F
  \}\\
  C_v &= \{
  \{(\tuple{v_i,a},\tuple{v_{i+1},b})\} \mapsto F
  \mid i=1,\ldots,r-1
  \}\\
  C_w &= \{
  \{(\tuple{w_i,a},\tuple{w_{i+1},b})\} \mapsto F
  \mid i=1,\ldots,r-1
  \}\\
  C_x &= \{
  \{(\tuple{x_i,a},\tuple{x_{i+1},b})\} \mapsto F
  \mid i=1,\ldots,r-1
  \}
\end{align*}
See Pattern~\ref{pat:pivot} for an example, \Pivot$(3)$.
\end{definition}
The following proposition characterises those sets of binary \unstructured\ negative
patterns which Theorem~\ref{thm:intractable} does not prove
intractable.
\begin{proposition}
\label{prop:pivot} Any connected binary \unstructured\ negative
pattern $\chi$ either contains \Cycle$(k)$ (for some $k \geq
2$), \Valency, \Path, or \Valency+\Path, or it itself occurs in \Pivot$(r)$
for some integer $r \leq |\chi|$.
\end{proposition}
\begin{proof}
Suppose $\chi$ does not contain any of the patterns \Valency, \Cycle$(k)$ (for any $k \geq
2$), \Path, or \Valency+\Path. Since $\Cycle(2) \notoccursin \chi$,
each constraint pattern in $\chi$ contains at most one inconsistent
tuple. Recall that the \emph{valency} of a variable $x$ is the number
of distinct variables which share a constraint with $x$. Since $\chi$
does not contain \Valency\, it can only contain one variable of valency
three and all other variables must have valency at most two. Moreover, since
$\Cycle(k) \notoccursin \chi$ for $k \geq 3$, the negative structure
graph of $\chi$ does not contain any cycles. Then, since $\chi$ is
connected, the negative structure graph of $\chi$ consists of up to
three disjoint paths joined at a single vertex. The fact that $\chi$
does not contain \Path\ means there can be at most one pair of
intersecting inconsistent tuples in $\chi$ and, moreover, the fact
that $\chi$ does not contain \Valency+\Path\ means that this
intersection must occur on the variable with valency $3$, if it
exists. It then follows that $\chi$ must occur in \Pivot$(r)$, for
some $r \leq |\chi|$.
\end{proof}

\begin{corollary} \label{cory:pivot}
Let ${\mathcal X}$ be a set of connected binary \unstructured\ negative patterns.
Then \cspneg{\mathcal X} is tractable only if
there is some $\chi \in \mathcal X$ that occurs in $\Pivot(r)$, for some
integer $r \leq |\chi|$.
\end{corollary}

For an arbitrary (not necessarily \unstructured\ or negative) binary
CSP pattern $\chi$, we denote by {\sc neg}$(\chi)$ the flat negative
pattern obtained from $\chi$ by replacing all truth-values $T$ by $U$
in all constraint patterns in $\chi$ and ignoring the context. For a
set of patterns $\mathcal X$, {\sc neg}($\mathcal X$) is naturally defined as
the set {\sc neg}($\mathcal X$) $= \{${\sc neg}$(\chi) : \chi \in {\mathcal
X} \}$. Clearly \cspneg{{\text {\sc neg}}(\mathcal X)} $\subseteq$
\cspneg{\mathcal X}. The following result follows immediately from
Corollary \ref{cory:pivot}.

\begin{corollary}
Let ${\mathcal X}$ be a set of binary patterns such that for each 
$\chi \in \mathcal X$, {\sc neg}$(\chi)$ is connected. Then
\cspneg{\mathcal X} is tractable only if there is some $\chi \in \mathcal X$
such that {\sc neg}$(\chi)$ occurs in $\Pivot(r)$, for some integer
$r \leq |\chi|$.
\end{corollary}

\begin{pattern}\vspace{0.3cm}
\centering
\begin{tikzpicture}[scale=0.8]
  \node at (0.5,6.3) {$v_3$};
  \draw[rounded corners=12pt] (0,6) rectangle (1,4);
  \fill (0.5,5.5) circle (0.1);
  \fill (0.5,4.5) circle (0.1);
  \draw[dashed] (0.5,4.5) -- (2.25+0.5,5.5);

  \node at (2.75,6.2) {$v_2$};
  \draw[rounded corners=12pt] (2.25,6) rectangle (2.25+1,4);
  \fill (2.25+0.5,5.5) circle (0.1);
  \fill (2.25+0.5,4.5) circle (0.1);
  \draw[dashed] (2.25+0.5,4.5) -- (2*2.25+0.5,5.5);

  \node at (5,6.2) {$v_1$};
  \draw[rounded corners=12pt] (4.5,6) rectangle (4.5+1,4);
  \fill (2*2.25+0.5,5.5) circle (0.1);
  \fill (2*2.25+0.5,4.5) circle (0.1);
  \draw[dashed] (2*2.25+0.5,4.5) -- (3*2.25+0.5,4.5);

  \node at (0.5,2.2) {$x_3$};
  \draw[rounded corners=12pt] (0,2) rectangle (1,0);
  \fill (0.5,1.5) circle (0.1);
  \fill (0.5,0.5) circle (0.1);
  \draw[dashed] (0.5,0.5) -- (2.25+0.5,1.5);

  \node at (2.75,2.2) {$x_2$};
  \draw[rounded corners=12pt] (2.25,2) rectangle (2.25+1,0);
  \fill (2.25+0.5,1.5) circle (0.1);
  \fill (2.25+0.5,0.5) circle (0.1);
  \draw[dashed] (2.25+0.5,0.5) -- (2*2.25+0.5,1.5);

  \node at (5,2.2) {$x_1$};
  \draw[rounded corners=12pt] (4.5,2) rectangle (4.5+1,0);
  \fill (2*2.25+0.5,1.5) circle (0.1);
  \fill (2*2.25+0.5,0.5) circle (0.1);
  \draw[dashed] (2*2.25+0.5,0.5) -- (3*2.25+0.5,2.5);

  \node at (3*2.25+0.5,5+0.5) {$p$};
  \draw[rounded corners=12pt] (3*2.25,5) rectangle (3*2.25+1,2);
  \fill (3*2.25+0.5,4.5) circle (0.1);
  \node at (3*2.25+0.5+0.7,4.5) {$a$};
  \fill (3*2.25+0.5,3.5) circle (0.1);
  \node at (3*2.25+0.5-0.7,3.5) {$b$};
  \fill (3*2.25+0.5,2.5) circle (0.1);
  \node at (3*2.25+0.5+0.7,2.5) {$c$};
  \draw[dashed] (3*2.25+0.5,3.5) -- (4*2.25+0.5,2.5+1);

  \draw[dashed] (3*2.25+0.5,3.5) -- (4*2.25+0.5,2.5+1);

  \node at (4*2.25+0.5,4+1+0.2) {$w_1$};
  \draw[rounded corners=12pt] (4*2.25,4+1) rectangle (4*2.25+1,2+1);
  \fill (4*2.25+0.5,3.5+1) circle (0.1);
  \fill (4*2.25+0.5,2.5+1) circle (0.1);
  \draw[dashed] (4*2.25+0.5,3.5+1) -- (5*2.25+0.5,2.5+1);

  \node at (5*2.25+0.5,4+1+0.2) {$w_2$};
  \draw[rounded corners=12pt] (5*2.25,4+1) rectangle (5*2.25+1,2+1);
  \fill (5*2.25+0.5,3.5+1) circle (0.1);
  \fill (5*2.25+0.5,2.5+1) circle (0.1);
  \draw[dashed] (5*2.25+0.5,3.5+1) -- (6*2.25+0.5,2.5+1);

  \node at (6*2.25+0.5,4+1+0.2) {$w_1$};
  \draw[rounded corners=12pt] (6*2.25,4+1) rectangle (6*2.25+1,2+1);
  \fill (6*2.25+0.5,3.5+1) circle (0.1);
  \fill (6*2.25+0.5,2.5+1) circle (0.1);

\end{tikzpicture}

$$NEQ(p,v_1,v_2,v_3,w_1,w_2,w_3,x_1,x_2,x_3)$$

\caption{\SepPivot$(3)$}
\label{pat:seppivot}
\end{pattern}

We now define a pattern we call a {\em separable pivot}; forbidding
this pattern defines a subclass of \cspneg{\Pivot(r)}.
\begin{definition}
\label{defn:seppivot}
Let $V = \{p\} \cup \{v_1,\ldots,v_r\} \cup \{w_1,\ldots,w_r\} \cup
\{x_1,\ldots,x_r\}$ and let $D = \{a,b,c\}$. We define the pattern
$\SepPivot(r)=(V,D,C_p \cup C_v \cup C_w \cup C_x)$, where
\begin{align*}
  C_p & = \{
  \{(\tuple{p,a},\tuple{v_1,b})\}\mapsto F,
  \{(\tuple{p,b},\tuple{w_1,b})\}\mapsto F,
  \{(\tuple{p,c},\tuple{x_1,b})\}\mapsto F
  \}\\
  C_v &= \{
  \{(\tuple{v_i,a},\tuple{v_{i+1},b})\} \mapsto F
  \mid i=1,\ldots,r-1
  \}\\
  C_w &= \{
  \{(\tuple{w_i,a},\tuple{w_{i+1},b})\} \mapsto F
  \mid i=1,\ldots,r-1
  \}\\
  C_x &= \{
  \{(\tuple{x_i,a},\tuple{x_{i+1},b})\} \mapsto F
  \mid i=1,\ldots,r-1
  \}
\end{align*}
See Pattern~\ref{pat:seppivot} for an example, \SepPivot$(3)$.
\end{definition}

Clearly, \SepPivot$(r)$ occurs in
\Pivot$(r)$: we take a bijection between corresponding variable-value pairs for
the $v_i$, $w_i$ and $x_i$ variables, map both \tuple{p,a} and
\tuple{p,b} in \SepPivot$(r)$ to \tuple{p,a} in \Pivot$(r)$, and map
\tuple{p,c} in \SepPivot$(r)$ to \tuple{p,b} in \Pivot$(r)$. We will now
show that \cspneg{\SepPivot(r)} is tractable for any fixed $r$.
\begin{theorem} \label{thm:seppivottract}
Let $r$ be a fixed integer. $\cspneg{\SepPivot(r)}$
can be solved in polynomial time.
\end{theorem}
\begin{proof}
  
  By the grid minor theorem of Robertson and Seymour
  \cite{Robertson1986:graphminorsV}, there exists a function $f : \mathbb{N}
  \rightarrow \mathbb{N}$ such that any graph $G$ with tree width $tw(G) \geq
  f(r)$ must contain an $r \times r$ grid as a
  minor. 

  Now \SepPivot$(r)$ occurs in any CSP instance whose constraint
  graph contains a vertex that starts three disjoint paths. Certainly,
  any CSP instance P whose constraint graph contains an $r \times r$
  grid as a minor will satisfy this condition. Hence, $P \in
  \cspneg{\SepPivot(r)}$ is only possible when $tw(P) < f(r)$. Since
  the class of CSP instances with tree width bounded above by $f(r)$
  is tractable, it follows that $\cspneg{\SepPivot(r)}$ is
  tractable.\footnote{The best upper bound on the function $f(r)$ is
    superexponential: $20^{2r^5}$ \cite{Robertson1994:quickly}. Thus,
    Theorem~\ref{thm:seppivottract} does not actually provide a
    practical algorithm for solving problems in
    \cspneg{\SepPivot(r)}.}.
\end{proof}

The following corollary is a direct consequence of Lemma \ref{lem:nomaximal}.

\begin{corollary}\label{prop:seppivotuniontract}
  Any disjoint union of \SepPivot$(r)$ patterns is a tractable pattern.
\end{corollary}

Next, we show that forbidding $\Pivot(1)$ gives rise to a tractable
class of CSPs.
\begin{proposition}
  \label{prop:pivot1}
  $\cspneg{\Pivot(1)}$ can be solved in polynomial time.
\end{proposition}
\begin{proof}
We will show that every $P \in \cspneg{\Pivot(1)}$ can be reduced in
polynomial time to
$P' \in \cspneg{\negtrans}$ such that $P$ has a solution if and only
if $P'$ has a solution. Without loss of generality, we assume that
$P$ is arc-consistent since eliminating domain values (by arc consistency)
cannot destroy the fact that $P$ is \Pivot$(1)$-free.

Suppose $\negtrans$ occurs in $P$ at $\{u,p,v\}$:
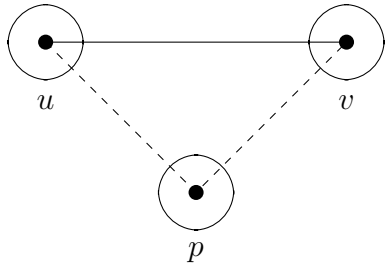
\begin{center}
\begin{tikzpicture}
  \path[draw,rounded corners=15pt] (0,3) rectangle (1,2);
  \fill (0.5,2.5) circle (0.1);
  \node at (0.5,1.7) {$u$};

  \path[draw,rounded corners=15pt] (2,1) rectangle (3,0);
  \fill (2.5,0.5) circle (0.1);
  \node at (2.5,-0.3) {$p$};

  \path[draw,rounded corners=15pt] (4,3) rectangle (5,2);
  \fill (4.5,2.5) circle (0.1);
  \node at (4.5,1.7) {$v$};

  \draw (0.5,2.5) -- (4.5,2.5);
  \draw[dashed] (0.5,2.5) -- (2.5,0.5) -- (4.5,2.5);
\end{tikzpicture}
\end{center}
Since $P$ does not contain $\Pivot(1)$, it follows that $p$ cannot be
connected to any variables other than $u,v$ in the constraint graph
of $P$. Thus, we can obtain an equivalent CSP $P_1$ by eliminating
$p$ and tightening the constraint on $\{u,v\}$ by disallowing any
assignment which does not extend to an assignment of $p$. This new
CSP is still \Pivot$(1)$-free but has had the occurrence of
\negtrans\ on $\{u,p,v\}$ eliminated. To see that $P_1$ is
\mbox{\Pivot$(1)$-free}, suppose that the pair of assignments
$(\tuple{u,a},\tuple{v,b})$ becomes incompatible in $P_1$ after
elimination of variable $p$ from $P$. By arc consistency of $P$,
there are values $c,d$ such that the pairs
$(\tuple{u,a},\tuple{p,c})$, $(\tuple{v,b},\tuple{p,d})$ are
consistent in $P$. But, since $(\tuple{u,a},\tuple{v,b})$ cannot be
extended to an assignment of $p$ in $P$, this implies that the pairs
$(\tuple{u,a},\tuple{p,d})$, $(\tuple{v,b},\tuple{p,c})$ are
necessarily inconsistent in $P$. Now, if the inconsistent pair
$(\tuple{u,a},\tuple{v,b})$ were part of an occurrence of \Pivot$(1)$
in $P_1$, then we could easily obtain an occurrence of \Pivot$(1)$ in
$P$ by replacing $(\tuple{u,a},\tuple{v,b})$ by either
$(\tuple{u,a},\tuple{p,d})$ or $(\tuple{v,b},\tuple{p,c})$ (depending
on whether it is variable $u$ or $v$ which is at the centre of the
pivot).

Thus, by repeatedly identifying and eliminating occurrences of \negtrans, we
will eventually (after the elimination of at most $n-2$ variables)
obtain a CSP $P' \in \cspneg{\negtrans}$. By the way
we have constructed $P'$, we know that any solution to $P'$ can be
extended to an assignment on the removed variables. Thus, since we can
solve any instance $P' \in \cspneg{\negtrans}$ in polynomial time
(Theorem~\ref{thm:negtrans}), we can solve any instance $P$ from 
$\cspneg{\Pivot(1)}$ in polynomial time.
\end{proof}
Proposition~\ref{prop:pivot1} is important as it gives us a tractable
class of CSPs defined by forbidding a negative pattern which, unlike
$\cspneg{\SepPivot(r)}$, contains problems of
unbounded tree width, and so cannot be captured by structural
tractability. As an example of a class of CSP instances in
$\cspneg{\Pivot(1)}$ with unbounded tree width, consider the
$n$-variable CSP instance $P_n$ with domain $\{1,\ldots,n\}$ whose
constraint graph is the complete graph and, for each pair of distinct
values $i,j \in \{1,\ldots,n\}$, the constraint on variables
$v_{i},v_{j}$ disallows a single pair of assignments
$(\tuple{v_{i},j},\tuple{v_{j},i})$. Since each assignment
$\tuple{v_{i},j}$ occurs in a single inconsistent tuple, $\Pivot(1)$
does not occur in $P_n$, and hence $P_{n} \in \cspneg{\Pivot(1)}$.

We conjecture that there exists a larger
class of tractable problems defined by forbidding negative patterns.

\begin{conjecture}\label{conj:pivot}
For a fixed integer $r$, $\cspneg{\Pivot(r)}$ can be solved in
polynomial time.
\end{conjecture}

A positive answer to Conjecture~\ref{conj:pivot}, taken in
conjunction with Corollary~\ref{cory:pivot}, would give a dichotomy
result for CSPs defined by forbidding a finite set of binary
\unstructured\ negative patterns, which we state in the form of a
conjecture.

\begin{conjecture}\label{conj:dich}
Let ${\mathcal X}$ be a finite set of connected binary \unstructured\
negative patterns. Then \cspneg{\mathcal X} is tractable if and only if
there is some $\chi \in \mathcal X$ that is contained in \Pivot$(r)$, for
some integer $r \leq |\chi|$.
\end{conjecture}

\section{Conclusion}
\label{sec:conclusions}

In this paper we described a framework for identifying classes of
CSPs in terms of {\em forbidden patterns}, to be used as a tool for
identifying tractable classes of the CSP. We gave several examples of
small patterns that can be used to define tractable classes of CSPs.

In the search for a general result, we restricted ourselves to the
special case of binary \unstructured\ negative patterns. In
Theorem~\ref{thm:intractable} we showed that $\cspneg{\mathcal
  X}$ is NP-hard if every pattern in a set $\mathcal X$ contains at
least one of four patterns
(Patterns~\ref{pat:cycle},~\ref{pat:valency},~\ref{pat:path},
and~\ref{pat:valpath}). Moreover, we showed that any connected binary
\unstructured\ negative pattern that did not contain any of these
patterns must itself be contained within a special type of pattern
called a \emph{pivot}. Hence, the presence of a pivot is a necessary
condition for pattern $\chi$ to be tractable. We were able to show
that another pattern, which we call {\em separable pivot}, occurs
in the pivot pattern and defines a tractable class. Hence,
separable pivots define a tractable subclass of the class defined by
pivots. We conjecture that tractability extends to the whole class of
problems defined by pivots. We proved tractability for pivots of size
$1$.

The main open problem is the resolution of the tractability of pivots
of any fixed size $r$. Beyond this, it will be interesting to see
what new tractable classes can be defined by more general binary
patterns or by non-binary patterns. In particular, an important area
of future research is determining all maximal tractable classes of
problems defined by patterns of some fixed size (given by the number
of variables or the number of variable-value assignments).

\vskip 0.2in
\bibliographystyle{plainurl}

\end{document}